\documentclass[11pt,DIV=12,BCOR=0.7cm]{article}   

\usepackage[left=2cm, right=2cm, top=2cm, bottom=3cm]{geometry}
\usepackage{graphicx,amssymb}
\usepackage{amsmath,amsfonts, amsthm}
\usepackage{bm}
\usepackage{mathtools,mathdots}
\usepackage{subcaption}
\usepackage{algorithm,algorithmic}
\usepackage{setspace}
\usepackage{scalerel,stackengine}
\usepackage{mathtools}
\usepackage{tikz}
\usetikzlibrary{matrix,chains,positioning,decorations.pathreplacing,arrows}
\usetikzlibrary{positioning,calc}
\usepackage{url}
\usepackage{wrapfig, multirow}
\usepackage{tablefootnote}
\usepackage{textcomp}
\usepackage[hidelinks]{hyperref}
\usepackage{makecell}
\usepackage{arydshln}
\usepackage{accents}

\newtheorem{lemma}{Lemma}

\newtheorem{theorem}{Theorem}[section]

\newcommand*\diff{\mathop{}\mathrm{d}}
\DeclarePairedDelimiterX{\infdivx}[2]{\big(}{\big)}{%
	#1\;\delimsize\|\;#2%
}
\newcommand{\infdiv}{\text{KL}\infdivx}
\DeclarePairedDelimiter{\norm}\lVert\rVert

\DeclareMathOperator{\E}{\mathbb{E}}
\newcommand\equalhat{\mathrel{\stackon[1.5pt]{=}{\stretchto{%
				\scalerel*[\widthof{=}]{\wedge}{\rule{1ex}{3ex}}}{0.5ex}}}}

\newcommand{\expnumber}[2]{{#1}\mathrm{e}{#2}}

\newcommand\thickbar[1]{\accentset{\rule{\widthof{#1}*\real{.75}}{.75pt}}{#1}}

\newcommand\Tstrut{\rule{0pt}{2.6ex}}         

\usepackage[normalem]{ulem}

\newenvironment{keywords}
{\bgroup\leftskip 20pt\rightskip 20pt \small\noindent{\bf Keywords:} }%
{\par\egroup\vskip 0.25ex}

\makeatletter
\DeclareRobustCommand{\Udots}{%
	\vcenter{\offinterlineskip
		\halign{%
			\hbox to .8em{##}\cr
			\hfil.\cr\noalign{\kern.2ex}
			\hfil.\hfil\cr\noalign{\kern.2ex}
			.\hfil\cr}%
	}%
}
\makeatother

\title{Complexity-Aware Training of Deep Neural Networks for Optimal Structure Discovery}
\author{Valentin Frank Ingmar Guenter\thanks{Email:
		vgunter@uci.edu}\ \ and\ \ Athanasios~Sideris\thanks{Email:
		asideris@uci.edu}\\ \\  Department of Mechanical and
	Aerospace Engineering,\\ University of California, Irvine,\\
	Irvine, CA, 92697}
\date{}

\begin{document}

\maketitle
\begin{abstract}
We propose a novel algorithm for combined unit and layer pruning of deep neural networks that functions during training and without requiring a pre-trained network to apply. Our algorithm optimally trades-off learning accuracy and pruning levels while balancing layer vs. unit pruning and computational vs. parameter complexity using only three user-defined parameters, which are easy to interpret and tune. To this end, we formulate a stochastic optimization problem over the network weights and the parameters of variational Bernoulli distributions for binary Random Variables taking values either 0 or 1 and scaling the units and layers of the network. Then, optimal network structures are found as the solution to this optimization problem.
Pruning occurs when a variational parameter converges to 0 rendering the corresponding structure permanently inactive, thus saving computations both during training and prediction. A key contribution of our approach is to define a cost function that combines the objectives of prediction accuracy and network pruning in a computational/parameter complexity-aware manner and the automatic selection of the many regularization parameters.
We show that the proposed algorithm converges to solutions of the optimization problem corresponding to deterministic networks. We analyze the ODE system that underlies our stochastic optimization algorithm and establish domains of attraction for the dynamics of the network parameters. These theoretical results lead to practical pruning conditions avoiding the premature pruning of units on each layer as well as whole layers during training.
We evaluate our method on the CIFAR-10/100 and ImageNet datasets using ResNet architectures and demonstrate that it gives improved results with respect to pruning ratios and test accuracy over layer-only or unit-only pruning and favorably competes with combined unit and layer pruning algorithms requiring pre-trained networks.
\end{abstract}


\section{Introduction}
Deep neural networks (DNNs) have been successfully employed to solve diverse learning problems such as pattern classification, reinforcement learning, and natural language processing
\cite{girshick_fast_2015, noh_learning_2015, silver_mastering_2017}.
In recent years, it has been widely accepted that large-scale DNNs generally outperform smaller networks when trained correctly.
Although residual architectures such as ResNets \cite{ResNet_he_2016,he_identitymappings2016} with number of layers well into the hundreds have been successfully trained with techniques such as Batch Normalization \cite{batchnorm_ioffe2015}, such large-scale and resource-consuming models are still challenging to train and more so to deploy if the available hardware is limited, e.g., on mobile devices.

To cope with the training and deployment issues facing large DNNs, various pruning methods have been developed to eliminate neural networks structures from such large-scale networks while mostly preserving their predictive performance \cite{han_learning_2015,blalock_what_2020}.
While at first such methods pruned individual weights \cite{han_learning_2015}, in more recent years structured unit pruning \cite{li_pruning_2017, he_soft_2018, liebenwein_lost_2021, SCOP_tang_2020, GUENTER2024_robust} and layer pruning \cite{chen19_shallowing, DBPwang2019, guenter_concurrent} have been found to be more effective. In addition, contributing to the success of pruning methods, it is often found that starting from larger networks and pruning them results in better performance than simply training smaller models from scratch
\cite{blalock_what_2020, frankle_lottery_2019, Louizos_2017, guenter_concurrent}.

However, most existing pruning methods come with certain limitations that make them inconvenient to use in practice.
For example, many methods require pre-trained networks and rely on expensive iterative pruning and retraining schemes or require significant effort in tuning hyperparameters.
Moreover, most pruning methods focus on either unit or layer pruning and existing literature on effectively combining the two is limited.
While in principle a variety of unit and layer pruning methods can be applied successively to obtain a NN that is pruned in both dimensions, it is often unclear to what extent layer pruning, which reduces the sequential workload of computing through layers, is more beneficial than unit pruning, which reduces the parallelizable computations within a layer. For example, and demonstrated by our experiments in Section~\ref{sec:experiments} (Figure~\ref{fig:imagenet_timing} and Table~\ref{tab:final_prediction_time}), focusing on unit pruning can fail to achieve comparable speed-ups to layer pruning for similar pruning ratios when parallelizable GPU computing is used. Moreover, the effect of pruning on reducing the computational load or floating point operations (FLOPS) vs. the number of network parameters and the memory footprint of a NN needs to be distinguished and regulated.

\subsection{Overview of the Proposed Algorithm and Summary of Contributions}
In this work, we combine the unit pruning from \cite{GUENTER2024_robust} and the layer pruning from \cite{guenter_concurrent} to effectively address  the raised difficulties of balancing layer vs. unit pruning as well as FLOPS vs. parameter pruning.
Similar to \cite{GUENTER2024_robust,guenter_concurrent}, we introduce Bernoulli RVs multiplying the outputs of units and layers of the network and use a projected gradient descent algorithm to learn the hyperparameters of their postulated Bernoulli variational posterior distributions.
In this manner, we obtain a novel algorithm for complete structure discovery of DNNs capable of producing a compact sub-network with state-of-the art performance via a single training run that does not require pre-trained weights or iterative re-training.

There are several key contributions in this work distinguishing it from and further developing the previous literature.
While the approaches of \cite{GUENTER2024_robust} and \cite{guenter_concurrent} work well and robustly for a wide range of the hyperparameters of the scale Bernoulli RVs when taken to be common in all layers, it was observed that choosing them layer-wise throughout the network improves performance. However, for deep NNs with many layers tuning these hyperparameters by hand becomes infeasible. To overcome this issue, the algorithm proposed here provides a  novel approach to automatically select such layer-wise regularization parameters based on the expected computational and memory complexity of the neural network. More specifically, we formulate an optimization objective based on a Bayesian approach for trading-off network performance and complexity in which
we relate the expected number of FLOPS as well as the number of parameters of the network to the regularization terms in the objective function involving the
scale RVs parameters of units and layers. This optimization objective explicitly integrates the objectives of prediction accuracy and of combined layer and unit pruning in a complexity-aware manner. Thus, our approach furnishes a framework to determine optimal NN structures from both a performance and computational complexity and size objectives.  The proposed method relies on only three easily interpretable hyperparameters that i) control the general pruning level, ii) regulate layer vs. unit pruning and iii) trade-off computational load vs. the number of parameters of the network. The regularization parameters for units and layers are chosen dynamically during training in terms of these three hyperparameters addressing all previously discussed challenges of combined layer and unit pruning.

Furthermore, the theoretical results in \cite{GUENTER2024_robust, guenter_concurrent} of their respective training/pruning algorithms, which was used to justify pruning conditions in the sense that eliminated units cannot recover with further training, is shown to fail when layer and unit pruning are combined, necessitating a different analysis.
Specifically, in the setting of combined layer and unit pruning convergence needs to be established to a set, representing possible solutions of the optimization problem corresponding to either layer or unit pruning, rather than a single equilibrium point.
Using Lyapunov Stability theory, we prove conditions under which the scale RVs and the weights of a unit or layer of the NN converge to their elimination and can be safely removed
during training (Theorem~\ref{thm:stab_layer} and Theorem~\ref{thm:stab_unit}).
These theoretical results lead to practical pruning conditions implemented in our algorithm. In addition, we show that the resulting pruned networks are always deterministic and therefore extend one of the main results of \cite{guenter_concurrent} to the combined pruning case (Theorem~\ref{thm:deterministic_network}).

The remainder of the paper is organized as follows. We summarize the statistical approach and derive the optimization problem that underpins the proposed pruning method in Section~\ref{sec:Problem_Formulation}.
Section~\ref{sec:cost_aware} presents the main ideas in combining unit and layer pruning to address its challenges and provides the theoretical characterization of the solutions of the optimization problem.
Section~\ref{sec:algorithm} gives a summary of our concurrent learning/pruning algorithm. Section~\ref{sec:convergence_results} provides theoretical results on the convergence properties of the algorithm, with their proofs relegated to the Appendices.
Section~\ref{sec:related_work} reviews the related literature.
In Section~\ref{sec:experiments}, we discuss extensive simulation results with commonly used deep neural network architectures and datasets, and demonstrate the effectiveness of the proposed method over competing approaches.
We give conclusions in Section~\ref{sec:conclusions}.
\newline
\newline
\noindent\emph{Notation:}
We differentiate parameters/variables associated with the $l$th layer of the neural network by using the superscript ``$l$'';  indexing such as $x(n)$ signifies iteration counting; we use $\norm{\cdot}$ exclusively for the Euclidean/Frobenius norm and $v^T/M^T$ for the transpose of a vector/matrix; $v_1\odot v_2$ denotes componentwise multiplication of vectors $v_1$, $v_2$. We denote by $\E[\cdot]$ taking expected value with respect to specified random variables.  Additional notations are defined as needed.

\section{Problem Formulation and Background Material}\label{sec:Problem_Formulation}
We consider a residual network structure obtained by cascading the blocks depicted in Figure~\ref{fig:NN_struct_specific}. The part of the block inside the dashed box can be repeated $M$ times to yield many commonly used structures, such as the various ResNet networks (in such cases, consecutive weights are combined and replaced by a single weight).
The theoretical part of this work assumes for simplicity fully-connected layers as shown and $M=1$, but the results can be readily extended to structures with $M>1$ and convolutional layers as discussed later and demonstrated by our simulations.
This structure results in a network with $L$ residual blocks realizing mappings $\hat y = NN(x;W,\Xi)$ as follows:
\begin{align}\label{eq:struct_res}
\begin{split}
z^1 &=x; \\
\bar z^l&=\xi_2^l\odot h^l(z^l);\\
z^{l+1}&=\xi_B^l\cdot \left[W^l_2 \left(\xi_1^l\odot a^l\left(W^l_1 \bar z^l+b_1^l\right)\right)\right] + W_3^l \bar z^l+b_3^l,\quad
\quad\text{for}\quad l=1, \dots, L;\\
\hat y &= h^{L+1}(z^{L+1}).
\end{split}
\end{align}

\begin{figure}[t!]
	\centering
	\includegraphics[width=0.9\textwidth]{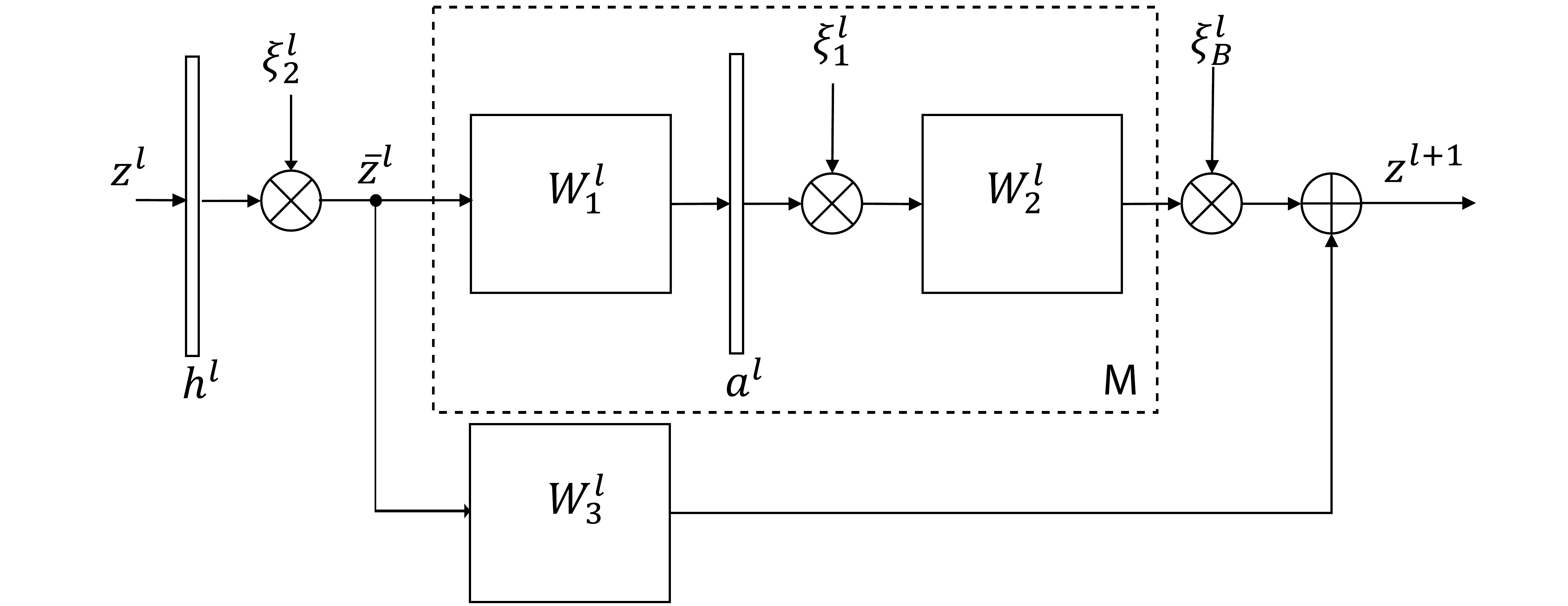}
	\caption{The network structure considered in this work (see \eqref{eq:struct_res}). $z^l$ are input features, $a^l, h^l$ the activation functions and $W_1^l, W_2^l$ and $W_3^l$ network weights. The scalar Bernoulli RV $\xi_B^l$ multiplies the output of the block $W_1^l-a^l-W_2^l$ and the vector valued Bernoulli RVs $\xi_1^l$ and $\xi_2^l$ multiply their corresponding signals element-wise. Bias weights are not shown for simplicity. The part of the block inside the dashed box can be repeated $M$ times to yield many commonly used structures.}\label{fig:NN_struct_specific}
\end{figure}
In our model \eqref{eq:struct_res}, $W_1^l, W_2^l, W_3^l, b_1^l$, $b_3^l$ are weight matrices of appropriate dimensions and $\xi_B^l$ are scalar binary variables taking values either 0 or 1. Similarly, $\xi_j^l$, $j\in\{1,2\}$ are vector valued binary variables with entries either 0 or 1.
We denote collectively  $W_1^l$, $W_2^l$ and $W_3^l$  by $W$, $\xi_B^l$ by $\Xi_B$, and $\xi_j^l$  by  $\Xi_j$ for  $l=1,\ldots,L$ and $j\in\{1,2\}$. Lastly $\Xi_B, \Xi_1, \Xi_2$  are denoted together by $\Xi$. We assume that the element-wise activation functions $a^l(\cdot)$ and $h^l(\cdot)$ are twice continuously differentiable and satisfy $a^l(0)=0$ for $l=1,\dots, L$ and $h^l(0)=0$ for $l=1,\dots, L+1$.
To obtain a non-residual output layer, we can remove the nonlinear path of the last block leading to $\hat y = h^{L+1}(W_3^{L}\bar z_{L}+b_3^L)$. $h^{1}(\cdot)$ is typically taken to be the identity function but it is included in the formulation for added flexibility.
To simplify the notation, we absorb the additive bias vectors $b^l_1$ and $b^l_3$ into the weight matrices $W^l_1$ and $W^l_3$, respectively as an additional last column and append their inputs with a unity last element.

Realizations of $\xi_1^l$ and $\xi_2^l$ implement sub-networks of structure~\eqref{eq:struct_res} by removing the layer units corresponding to the components of $\xi_1^l$ and $\xi_2^l$ equal to 0.
Similarly, realizations of $\xi_B^l$ implement sub-networks of structure~\eqref{eq:struct_res} by removing the nonlinear path in layers with $\xi_B^{l} = 0$; specifically, $\xi_B^{l} = 0$ in block $l$ implies that its nonlinear path is inactive and its output is calculated from the linear path only, which potentially can be absorbed by an adjacent layer thus reducing the number of layers in the network.

In the remainder of this section, we follow the statistical approach of \cite{GUENTER2024_robust, guenter_concurrent} to formulate an optimization problem that allows to learn appropriate weights $W$ together with parameters $\Xi$ in our NN structure~\eqref{eq:struct_res}. The goal is to obtain a highly reduced sub-network defined by the $\Xi$'s not equal to $0$ achieving performance on par with the baseline unreduced network. In \cite{GUENTER2024_robust}, eliminated substructures are individual units only and in \cite{guenter_concurrent}
are layers only, while this work combines both to obtain a complete pruning scheme that optimally trades them off on the basis of computational and memory cost savings as detailed in Section~\ref{sec:cost_aware}.

\subsection{Statistical Modeling}\label{sec:statistical_model}
We consider input patterns $x_i \in \mathbb{R}^m$ and associated output values $y_i \in \mathbb{R}^n$ forming our dataset $\mathcal{D}=\lbrace (x_i,y_i)\rbrace_{i=1}^N$
and assume a statistical model for the data $p(Y\mid X,W,\Xi)$, which typically in the case of regression/classification learning problems is taken to be the Gaussian/crossentropy probability distribution functions (PDFs), respectively. We consider the weights $W$ as Random Variables (RV) and endow them with Normal prior distributions $p(W^l\mid \lambda_W)\equiv \mathcal N \big(0, \lambda_W^{-1} \mathbf{I}\big)$, $l=1,\ldots,L$ with $\lambda_W>0$.  We also consider the parameters $\Xi$ to be RVs: the scalar $\xi_B^l$ take Bernoulli prior distributions denoted by $p(\xi_B^l\mid \pi_B^l)$ with parameters $\pi_B^l$; the $\xi_j^l$ are assumed to be vectors of independent Bernoulli RVs with parameters given by the corresponding components of vectors $\pi_j^l$ and PDFs denoted by $p(\xi_j^l\mid \pi_j^l)$ for $l=1,\ldots,L$, $j\in\{1,2\}$.
We collect $\pi_B^l$ and $\pi_j^l$ for $l=1,\ldots,L$, $j\in\{1,2\}$  in $\Pi_B$  and $\Pi_j$, respectively and $\Pi_B$, $\Pi_j$ together in $\Pi$.

To encourage the convergence of elements of $\Pi$ and corresponding $\Xi$ to zero, and therefore the removal of layer features and whole layers in a manner avoiding
the premature pruning of network substructures \cite{GUENTER2024_robust},
we place {``flattening'' hyper-priors} $p(\pi_B^l \mid \gamma_B)$ and $p(\pi_j^l \mid \gamma_j), \,\, j\in\{1,2\}$  on the variables $\Pi$. The flattening hyper-prior was introduced in \cite{GUENTER2024_robust} and is defined as
\begin{align*}
p(\pi \mid \gamma) = \frac{\gamma-1}{\log\gamma}\cdot\frac{1}{1+(\gamma-1)(1-\pi)}, \quad 0<\gamma<1
\end{align*}
for $\pi=\pi_B^l$ or $\pi_j^l, \, j\in\{1,2\}$ , $l=1,\ldots, L$. A plot of this PDF for select values of the hyperparameter $\gamma$ is shown in Figure~\ref{fig:flatPDF}. We also define $p(\Pi|\Gamma)$ to be the product of such terms over the layers and units of the network with $\Gamma$ collecting the hyperparameters of the flattening PDFs.

\begin{figure}[t!]
	\centering
	\includegraphics[width=0.5\textwidth]{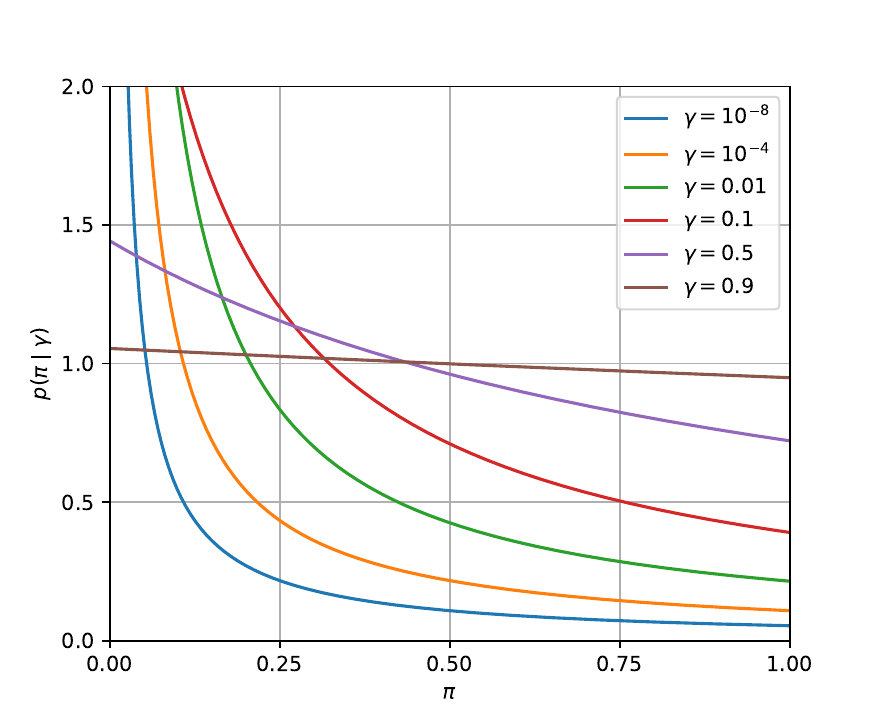}	
	\caption{Flattening PDF for select values of the hyperparameter $\gamma$.}
	\label{fig:flatPDF}
\end{figure}

\subsection{Variational Model Fitting}\label{sec:Variational_Approach}
By combining the NN's statistical model $p(Y\mid X,W,\Xi)$ with the prior distributions on the weights $W$ and parameters $\Xi$ and integrating out $\Xi$,
we arrive at the posterior
\begin{align*}
p(W,\Pi \mid Y,X,\lambda_W,\Gamma)
&\propto p(Y\mid X,W,\Pi)\cdot p(W\mid\lambda_W)\cdot p(\Pi\mid\Gamma).
\end{align*}
Then, we select the parameters $W$ and $\Pi$ via
Maximum a Posteriori (MAP) estimation:
\begin{align}\label{eq:obj_pi}
\max_{W, \Pi} \left[\log p(Y\mid X,W,\Pi) +\log p(W\mid \lambda_W)+ \log p(\Pi\mid \Gamma) \right],
\end{align}
with all $\pi$ in $\Pi$ constrained in $[0,\,1]$. In what follows, $W$ and $\Pi$ will denote point estimates rather than RVs.
We remark that the first term in objective \eqref{eq:obj_pi} promotes NN performance, the second term encourages generalization and the third term aims at reducing the size of the network.
Since $p(Y \mid X,W,\Pi)$ in  \eqref{eq:obj_pi} resulting from integrating out $\Xi$ is intractable,
we replace it with the Variational lower bound ({\it ELBO}) (see e.g. \cite[Chapter 10]{bishop_pattern_2006})
\begin{align*}
\log p(Y\mid X,W,\Pi)
\geq &\int q(\Xi\mid \Theta) \log p(Y\mid X,W,\Xi) \diff \Xi -
\infdiv{q(\Xi\mid \Theta)}{p(\Xi\mid \Pi)},
\end{align*}
where
$q(\Xi\mid\Theta)$ and $p(\Xi\mid\Pi)$ denote the product of Bernoulli variational distributions on the $\xi$ RVs with corresponding parameters $\theta$ collected in $\Theta$ and $\Pi$, respectively and
$KL(\cdot\mid\cdot)$ is the Kullback-Leibler divergence.
Thus, we obtain the maximization objective:
\begin{align}\label{eq:Objective_with_KL}
\begin{split}
\max_{W,\Pi,\Theta} \int q(\Xi\mid \Theta) \log p(Y\mid X,W,\Xi) \diff \Xi & - \infdiv{q(\Xi \mid \Theta)}{p(\Xi\mid \Pi)}
 + \log p(W\mid \lambda_W)  + \log p(\Pi\mid \Gamma),
\end{split}
\end{align}
where all $\pi$ in $\Pi$ and all $\theta$ in $\Theta$ are constrained to be in the interval $[0,\,1]$.
In the spirit of \cite{corduneanu_2001_model_sel} and \cite{GUENTER2024_robust}, we first maximize the main objective \eqref{eq:Objective_with_KL} with respect to $\Pi$ by equivalently minimizing:
\begin{align}\label{eq:maxobj_pi}
\min_\Pi \underbrace{\infdiv{q(\Xi\mid \Theta)}{p(\Xi\mid \Pi)} - \log p(\Pi\mid \gamma)}_{ \equiv J_{tot}(\Pi;\Theta) }.
\end{align}
Since $J_{tot}(\Pi;\Theta)$ in \eqref{eq:maxobj_pi} factorizes over the unit and layer probabilities, we can express $J_{tot}(\Pi;\Theta)=\sum_{(\pi,\theta)} J(\pi;\theta)$ with
\begin{align*}
J(\pi;\theta) = (1-\theta) \log \frac{1-\theta}{1-\pi} + \theta \log\frac{\theta}{\pi} + \log \left(1+(\gamma-1)(1-\pi)\right).
\end{align*}
It was shown in \cite{GUENTER2024_robust} for the flattening hyper-prior that the optimal $\pi$ is given by:
\begin{align}\label{eq:piflat}
\pi^\star = \pi^\star(\theta) = \arg\min_\pi J(\pi;\theta) =
\frac{\gamma\theta}{1+\theta(\gamma-1)},
\end{align}
which satisfies $0\leq \pi^\star\leq 1$ if $0\leq \theta\leq 1$ and also
\begin{align}\label{eq:flattening_J}
J(\pi^\star(\theta);\theta) = \log\gamma - \theta \log \gamma.
\end{align}
The reader is referred to \cite{GUENTER2024_robust} for the derivations of \eqref{eq:piflat} and \eqref{eq:flattening_J} but we remark here that the choice of the flattening hyperprior makes $\frac{\partial J(\pi^\star(\theta);\theta)}{\partial\theta}=-\log\gamma$ independent of $\theta$, i.e., ``flat'' and the pruning process easy to control via the selection of $\gamma$.

Next, we reparametrize $\lambda_W=\lambda\cdot N$ and write the weight regularization term as
\begin{align}\label{eq:W-prior}
\log p(W\mid \lambda_W) = const. -\frac{\lambda\cdot N}{2}W^\top W.
\end{align}
We also express the negative of the integral term in  \eqref{eq:Objective_with_KL} as
\begin{align} 
\E_{\substack{\Xi\sim  q(\Xi\mid \Theta)}}\big[-\log p(Y\mid X,W, \Xi)\big]
\approx N\cdot \E_{\substack{\Xi\sim  q(\Xi\mid \Theta) \\ (x_i,y_i) \sim p(\mathcal{D})}}\left[-\log p(y_i\mid x_i,W, \Xi)\right]= N\cdot C(W,\Theta),
\end{align}
where $p(\mathcal{D})$ represents the empirical distribution, which assigns probability $\frac{1}{N}$ to each sample $(x_i, y_i)$ in the dataset $\mathcal{D}=\lbrace (x_i,y_i)\rbrace_{i=1}^N$
and we defined
\begin{align}\label{eq:Cost}
C(W,\Theta) \equiv \E_{\substack{\Xi\sim  q(\Xi\mid \Theta) \\ (x_i,y_i) \sim p(\mathcal{D})}}\left[-\log p(y_i\mid x_i,W, \Xi)\right].
\end{align}

Finally, we insert \eqref{eq:flattening_J}, \eqref{eq:W-prior} and \eqref{eq:Cost} in \eqref{eq:Objective_with_KL} and after discarding constant terms, normalizing by the dataset size $N$ to express the cost per sample and equivalently shifting from maximization to minimization, we obtain the key form of the optimization problem in this work:
\begin{align}\label{eq:L_train_obj}
\begin{split}
\min_{W,\Theta} \quad  &
C(W,\Theta) + \frac{\lambda}{2}W^\top W -\frac{1}{N}\left[ \sum_{l=1}^{L}\theta_B^l \log\gamma_B^l +\sum_{l=1}^{L}\| \theta_1^l \|_1 \log\gamma_1^l + \sum_{l=1}^{L}\| \theta_2^{l} \|_1 \log\gamma_2^{l}\right] \\[10pt]
\text{subject to} \quad &\Theta\in\mathcal{H}= \left\{\theta_j^l \bigm| 0\leq\theta_j^l\leq 1, \, j\in\{1,2,B\} \,\, ,\forall \,l=1,\dots, L\right\}
\end{split}
\end{align}
with the inequalities in the definition of $\mathcal{H}$ applied element-wise if $\theta_j^l$ is a vector.

\section{Main Results}
\subsection{Complexity-Aware Regularization}\label{sec:cost_aware}
The results of \cite{GUENTER2024_robust} and \cite{guenter_concurrent} show that for deep ResNet \cite{ResNet_he_2016} architectures tuning the $\gamma$ hyperparameter for each layer individually instead of choosing the same value for the whole network may be necessary to obtain a network with state-of-the-art performance in terms of accuracy loss and pruning ratios. However, in deep networks with many layers, tuning ``manually'' individual hyperparameters for each layer becomes computationally unmanageable. In addition, unit pruning and layer pruning affect the structure of the resulting pruned network differently. Specifically, layer pruning tends to create shallow nets with wide layers while unit pruning produces deep nets with narrow layers. When both are in play, it is not clear how to select the hyperparameters controlling the two forms of pruning to strike a balance between the length of the network and the width of its layers in a way leading to optimal network structures. Next, we detail our approach, which addresses both of the above issues.

While minimizing the first term in the cost function \eqref{eq:L_train_obj} improves learning accuracy and keeping the second term small promotes generalization, the function of minimizing the remaining three terms is to achieve some significant reduction in the parameters and computational cost of the network; in particular, the first of these terms addresses the cost due to the number of layers while the remaining two the cost due to the number of units. In the following, we explicitly model parameter complexity and computational load and match it to the regularization terms on $\Theta$ in \eqref{eq:L_train_obj}. In this manner, minimization of these terms optimally balances layer/unit pruning towards a network structure that optimally trades-off learning accuracy versus computational cost and parameter complexity. Moreover, the hyperparameters $\gamma_1^l, \gamma_2^l$ and $\gamma_B^l$ for unit and layer pruning in each layer will be automatically determined as functions of $\Theta$ and only three user-selected hyperparameters, thus alleviating the burden of tuning each layer's $\gamma_j^l$, $j\in\{1,2,B\}$ individually. Furthermore, we show at the end of this section that the resulting form of the $\Theta$ regularization terms under the automatic selection of the $\gamma$ parameters maintains the desirable property shown in \cite{guenter_concurrent} that solutions of \eqref{eq:L_train_obj} are deterministic networks with all $\theta$ variables either at $0$ or $1$.

We first consider the computational cost of one pass through a section of the network depicted in Figure~\ref{fig:NN_struct_specific} for $M=1$. This cost is random depending of the RV's $\xi^l_j$, $j\in\{1,2,B\}$ that determine whether the non-linear path and particular units are active and in turn the effective size of the weights/biases involved. Thus, the effective column and row size of weight $W_1^l$ are given by $\sum_i\xi_{1i}^l$ and $\sum_i\xi_{2i}^l$, respectively, where the RVs $\xi_{ji}^l \sim Bernoulli(\theta_{ji}^l)$ control which input and output units are active during the pass. Then, we express the effective floating point operations (FLOPS) of the multiplication with weight $W_1^l$ as
\begin{align*}
f_1^l\left[ \xi_B^l\cdot\left(\sum_i\xi_{1i}^l\right) \cdot \left(\sum_i\xi_{2i}^l\right)\right],
\end{align*}
where $f_1^l$ is a constant proportionality factor determined by the layer type. More specifically, for a fully connected layer, we take $f_1^l=1$ reflecting the nominal cost of a multiplication/accumulation operation and for a 2D-convolutional layer with output of width $W$ and height $H$ as well as a kernel size of $K\times K$, we use $f_1^l = W\cdot H \cdot K^2$ to capture the total FLOPS of the convolutional operation accounting not only for the dimensions but also for the parameter sharing in the layer. Similarly, we can express the FLOPS due to operating with weights $W_2^l$ and $W_3^l$ by $f_2^l\left[ \xi_B^l\cdot(\sum_i\xi_{1i}^{l+1}) \cdot (\sum_i\xi_{2i}^l)\right]$ and $f_3^l(\sum_i\xi_{2i}^l) \cdot (\sum_i\xi_{2i}^{l+1})$, respectively, where the constants $f_2^l$, $f_3^l$ are also determined by the type of the layer used. We also consider a cost $f_a^l\xi_B^l\sum\xi_{1i}^l$ to capture any FLOPS associated with adding the bias $b_1$ and evaluating the activation function $a^l(\cdot)$ in \eqref{eq:struct_res}. Thus, the total computational cost of one pass through the section can be expressed as:
\begin{align*}
\xi_B^l \left(f_1^l \sum\xi_{1i}^l \cdot \sum\xi_{2i}^l + f_a^l\sum\xi_{1i}^l + f_2^l\sum\xi_{1i}^l \cdot \sum\xi_{2i}^{l+1}\right) + f_3^l\sum\xi_{2i}^l\cdot \sum\xi_{2i}^{l+1}.
\end{align*}

Since the RV $\xi$ are independent, $\sum\xi_{1i}^l$ and $\sum\xi_{2i}^l$ have Poisson binomial distributions with means $\sum\theta_{1i}^l=\|\theta_1^l\|_1$ and $\sum\theta_{2i}^l=\|\theta_2^l\|_1$, respectively since $\theta_{ji}^l\geq 0$. Then, the expected computational cost is expressed as:
\begin{align}\label{eq:comp_cmpx}
\begin{split}
J^l_F &=   \theta_B^l \left(f_1^l \norm{\theta_1^l}_1 \cdot \norm{\theta_2^l}_1 + f_a^l\norm{\theta_1^l}_1 + f_2^l\norm{\theta_1^l}_1 \cdot \norm{\theta_2^{l+1}}_1 \right) + f_3^l\norm{\theta_2^l}_1\cdot \norm{\theta_2^{l+1}}_1.	
\end{split}
\end{align}
In \eqref{eq:comp_cmpx}, we fix $\theta_2^1=[1,\, \dots,\, 1]^T, \,\theta_2^{L+1}=[1,\, \dots,\, 1]^T$ and $\theta_B^{L}=0$ to model the FLOPS of the first ($l=1$) and last $l=L$ layers of the network.

Next, we consider model parameter complexity, which relates to the memory requirements for implementing the network. We notice that the effective number of parameters also depends on the effective row and column sizes of the weights. Then, a parallel analysis gives the expected number of active parameters in the $l$th block as:
\begin{align}\label{eq:param_cmpx}
\begin{split}
J^l_P &=  \theta_B^l \left(p_1^l \norm{\theta_1^l}_1 \cdot \norm{\theta_2^l}_1 + p_a^l\norm{\theta_1^l}_1 + p_2^l\norm{\theta_1^l}_1 \cdot \norm{\theta_2^{l+1}}_1 \right) + p_3^l\norm{\theta_2^l}_1\cdot \norm{\theta_2^{l+1}}_1
\end{split}
\end{align}
with layer-type dependent constants $p_1^l,p_2^l$ and $p_3^l$. More specifically, we take for fully-connected layers $p^l=1$ and for convolutional layers $p^l=K^2$, where $K$ is the kernel size. Also, $p_a^l=1$ since the number of parameters in the $l$th layer due to the biases equals $\|\theta_1^l\|_1$.
We note that in the case of a fully-connected NN, the number of parameters is identical to the number of FLOPS computed according to \eqref{eq:comp_cmpx} with $f_a^l=1$.
Further we remark that when $h^l(\cdot)$ in our structure \eqref{eq:struct_res} is the identity function and the $l$th nonlinear path has been pruned, $W_3^l$ may be absorbed into the weights $W_1^{l+1}$ and $W_3^{l+1}$ for increased computational/parameter savings, which are not reflected in \eqref{eq:comp_cmpx} and \eqref{eq:param_cmpx}. In this case, the above equations overestimate the complexity of the network.
Also, if $W_3^l$ is a fixed identity matrix, i.e., there is no operation in the skip-connection, we can set $f_3^l=0$ and $p_3^l=0$. We do the same for the case of cheap pooling and padding operations often used in ResNet \cite{ResNet_he_2016} architectures.

To effectively trade-off computational vs. parameter complexity and unit vs. layer pruning, we construct a weighted complexity index:
\begin{align}\label{eq:cmpx_measure}
\begin{split}
J_{FP} = \beta\frac{\sum_{l=1}^{L} J^l_F}{F} +(1-\beta) \frac{\sum_{l=1}^{L} J^l_P}{P} + \alpha \frac{\sum_{l=1}^{L}\cdot\theta_B^l}{L}.
\end{split}
\end{align}
In \eqref{eq:cmpx_measure}, the first term gives the expected computational cost of the pruned NN normalized by the total number of FLOPS $F$ of the unreduced NN per presentation; similarly the second term gives the expected number of parameters of the pruned NN normalized by the total number of parameters $P$ of the unreduced NN. After normalization, the computational and parameter measures become comparable quantities in $[0,\, 1]$ and can be traded-off with the coefficient $0\leq \beta\leq 1$. The third term in \eqref{eq:cmpx_measure} adds the expected number of layers in the pruned NN normalized by the number of layers of the unreduced network in $[0,\, 1]$ to make it comparable with the first two terms.  The coefficient $\alpha\geq 0$ can be used to further promote layer pruning vs. unit pruning and to distribute computations and parameters over shallow networks with more units per layer vs. deep networks with fewer units per layer. This choice should depend on whether the used computing hardware can take advantage of the potential for high parallelization in calculating many units of a layer rather than performing sequential operations across many narrower layers.

We can now express the network complexity as $\nu\cdot J_{FP}$ where $\nu>0$ is a proportionality coefficient and associate it with the last three terms in  \eqref{eq:L_train_obj} representing network complexity from a statistical viewpoint. More specifically, we equate
\begin{align} \label{eq:CA_equate}
\nu \cdot  J_{FP} =  -\frac{1}{N}\left[ \sum_{l=1}^{L}\theta_B^l \log\gamma_B^l +\sum_{l=1}^{L}\| \theta_1^l \|_1 \log\gamma_1^l + \sum_{l=1}^{L}\| \theta_2^{l} \|_1 \log\gamma_2^{l}\right]
\end{align}
and write the optimization objective \eqref{eq:L_train_obj} as
\begin{equation} \label{eq:L_train_obj_final}
\begin{split}
	&\min_{W,\Theta} \quad
	\underbrace{C(W,\Theta) + \frac{\lambda}{2}W^\top W + \nu \cdot  J_{FP}}_{\equalhat L(W,\theta)}\\
	&\text{subject to} \quad \Theta\in \mathcal{H},
\end{split}
\end{equation}
with $\mathcal{H}$ from \eqref{eq:L_train_obj} and defining the loss function $L(W,\Theta)$.
We remark that \eqref{eq:L_train_obj} derived from a statistical analysis and \eqref{eq:L_train_obj_final} derived from a complexity analysis are identical provided that \eqref{eq:CA_equate} holds; from \eqref{eq:cmpx_measure} and \eqref{eq:comp_cmpx}, \eqref{eq:param_cmpx}, we can readily obtain that \eqref{eq:CA_equate} is achieved by taking
\begin{align}\label{eq:gamma_choices}
\begin{split}
\log \gamma_B^l & = -\nu \alpha \cdot N \\
\log \gamma_1^l &=-\nu  \left({\theta_B^l \left(f_1^l\norm{\theta_2^l}_1 +
	f_a^l + f_2^l\norm{\theta_2^{l+1}}_1\right)} \frac{\beta}{F} +
{\theta_B^l \left(p_1^l\norm{\theta_2^l}_1 + 1 + p_2^l\norm{\theta_2^{l+1}}_1\right)}\frac{(1-\beta)}{P}\right)\cdot N
\\
\log \gamma_2^l &= -\nu \left( {f_3^l\norm{\theta_2^{l+1}}_1} \frac{\beta}{F}+ {p_3^l\norm{\theta_2^{l+1}}_1}\frac{1-\beta}{P}\right)\cdot N.
\end{split}
\end{align}
It is clear from \eqref{eq:L_train_obj_final} and \eqref{eq:comp_cmpx} that the network complexity can be transparently traded-off in the optimization by tuning the three hyperparameters $\nu$, $\beta$ and $\alpha$.
Specifically, $\nu$ can be used to trade-off learning accuracy vs. the overall pruning level, $\beta$ computational cost vs. parameter complexity, and $\alpha$ layer vs. unit pruning relative levels.
However, notice that $\log\gamma_j^l$, $j\in\{1,2,B\}$ are different for each layer and dynamically adjusted during training based on the evolving cost of the layer relative to its importance in the learning process.
More specifically, considering the gradients of $\nu \cdot J_{FP}$ with respect to $\theta_B^l, \theta_{ji}^l$ shows how each structure in the network is effectively regularized. Due to the specific choices of $\log\gamma_1^l$ and $\log\gamma_2^l$ in \eqref{eq:gamma_choices} and their dependency on $\theta_B^l$, each network block  is implicitly regularized by the number of active units in it, in addition to receiving regularization $\alpha$ from $\log\gamma_B^l$.
On the other hand, individual units receive generally more regularization in more active layers, i.e., when $\theta_B^l$ is closer to $1$.
Moreover, the regularization on units is also higher when many units in neighboring layers are active, for example large $\norm{\theta_2^l}_1$ and $\norm{\theta_2^{l+1}}_1$ lead to a larger $\log\gamma_1^l$.
\newline
\newline
\noindent\emph{Optimal Solutions:}
It was shown in \cite{guenter_concurrent} for a layer-wise dropout network  that the optimal solutions of the layer-wise version of \eqref{eq:L_train_obj} are deterministic networks with either $\theta_B^l=0$ or $1$. Hence, during the initial training phase when the $\theta$ values are in between $0$ and $1$, such a network enjoys the regularizing properties of dropout \cite{Srivastava_2014_Dropout}, while at the end of training when the $\theta$ parameters have converged, a well regularized and optimally pruned deterministic network is obtained.
Here, we extend the result from \cite{guenter_concurrent} to the case of combined layer and unit pruning.

\begin{theorem}\label{thm:deterministic_network}
	Consider the Neural Network \eqref{eq:struct_res} and the optimization problem \eqref{eq:L_train_obj_final}. The minimal values of
	\eqref{eq:L_train_obj_final} with respect to $\Theta$ are achieved at the extreme values $\theta_B^l=0,1$ and $\theta_{ji}^l=0,\, 1,\, j\in\{1,2\}$ and the resulting optimal network is deterministic.
\end{theorem}
\begin{proof}
  See \ref{app:proof_deterministic}.
\end{proof}

\subsection{Proposed Algorithm}\label{sec:algorithm}
The proposed concurrent learning and pruning algorithm to minimize the cost function~\eqref{eq:L_train_obj_final} is based on projected stochastic gradient descent. More specifically, we generate sequences \sloppy
$\{W^l_1(n), W^l_2(n), W_3^l(n), \theta_j^l(n)\}_{n\geq 0}$, $j\in\{1,2,B\}$ for  the nonlinear and linear paths:
\begin{align}\label{eq:DE_W}
\begin{split}
W_j^l(n+1) &= W_j^l(n)-a(n)\widehat{\frac{\partial L}{\partial W_j^l}},
\quad l=1,\dots, L, \quad j\in\{1,2,3\} \\
\end{split}
\end{align} and
\begin{align}\label{eq:DE}
\begin{split}
\theta_j^l(n+1) =
\max\left\{\min\left\{\theta_j^l(n)-a(n)\widehat{\frac{\partial L} {\partial \theta_j^l}},\,1\right\}, \,0\right\}, \quad l=1,\dots, L, \quad j\in\{1,2,B\}
\end{split}
\end{align}
ensuring that all $\theta_j^l$ remain in $\mathcal{H}$ and with stepsize $a(n)> 0$.

In the above equations, $\widehat{\frac{\partial L}{\partial W_j^l}}$ and $\widehat{\frac{\partial L} {\partial \theta_j^l}}$ are estimates of the gradients of the objective function $L(W(n),\Theta(n))$ in \eqref{eq:L_train_obj_final} with respect to network weights $W$ and parameters $\Theta$.
More specifically, let
\begin{align}
\hat C(W, \hat \Xi) = - \frac{1}{N}\log p(\hat Y\mid \hat X,W, \hat \Xi)
\end{align}
be a mini-batch sample of $C(W,\Theta)$ from \eqref{eq:Cost} using data $\hat X$, $\hat Y$ and samples $\hat \Xi\sim Bernoulli(\Theta)$, where the latter notation means that each scalar element $\xi\in\Xi$ and corresponding $\theta\in\Theta$ satisfy $\xi\sim Bernoulli(\theta)$.
Then,
\begin{align}\label{eq:W_grad}
\begin{split}
\frac{\partial L(W,\Theta)}{\partial W_j^l} \approx \frac{\widehat{\partial L(W,\Theta)}}{\partial W_j^l} = \frac{\partial \hat C}{\partial W_j^l} + \lambda W_j^l, \quad j\in\{1,2,3\}
\end{split}
\end{align}
and $\frac{\partial \hat C}{\partial W_j^l}$ can be calculated via standard backpropagation through a dropout network.
Further, the estimated gradients of $L(W,\Theta)$ with respect to $\Theta$ are obtained as follows. First, we derive from \eqref{eq:L_train_obj_final}:
\begin{align}\label{eq:Lgrads_theta}
\begin{split}
\frac{\partial L}{\partial \theta_{j}^l} = \frac{\partial C}{\partial  \theta_j^l} +\nu \frac{\partial J_{FP}}{\partial \theta_j^l}, \quad j\in\{1,2,B\}.
\end{split}
\end{align}
Next, we introduce the following notation to express conditioning of the expected value on realizations of the Bernoulli RVs $\xi_B^l$ and $\xi_{ji}^l,\, j=1$ or $2$ and corresponding to the $i$th nonlinear unit in layer $l$:
\begin{align}\label{eq:C_notation}
C_{ji, \hat \xi^l_{ji}}^{l, \hat \xi_B^l} = \E_{\substack{\Xi\sim  q(\Xi\mid \Theta)}}\left[-\frac{1}{N}\log p(Y\mid X,W, \Xi) \mid \xi_B^l = \hat \xi_B^l,\, \xi_{ji}^l = \hat \xi_{ji}^l \right].
\end{align}
Note that in \eqref{eq:C_notation} the samples $\hat\xi_B^l$ and $\hat\xi_{ji}^l$ take values $0$ or $1$; if conditioning on $\xi_{ji}^l$ is omitted,  the subscript notation is also omitted and if conditioning on $\xi_B^l$ is omitted, the value of the sample $\hat\xi_B^l$ is also omitted.
Then, we write $C(W,\Theta)$ in a way exposing its dependence on $\theta_B^l$ and $\theta_{ji}^l$, namely:
\begin{align}\label{eq:C_detail_notation}
\begin{split}
C(W,\Theta) = \theta_B^l C^{l,1} + \left(1-\theta_B^l\right) C^{l,0}
= \theta_B^l \left(\theta_{1i}^lC_{1i,1}^{l,1} +(1-\theta_{1i}^l)C_{1i,0}^{l,1}\right) + \left(1-\theta_B^l\right) C^{l,0},
\end{split}
\end{align}
and also
\begin{align*}
\begin{split}
C(W,\Theta) &= \theta_{2i}^l C_{2i,1}^l + \left(1-\theta_{2i}^l\right) C_{2i,0}^l.
\end{split}
\end{align*}
Differentiating these equations gives:
\begin{align}\label{eq:C_diffs}
\begin{split}
\frac{\partial C}{\partial \theta_B^l}=  C^{l,1} - C^{l,0}, \quad
\frac{\partial C}{\partial \theta_{1i}^l} = \theta_B^l \left(C_{1i,1}^{l,1} -C_{1i,0}^{l,1}\right)\quad \text{and}\quad
\frac{\partial C}{\partial \theta_{2i}^l} = C_{2i,1}^l - C_{2i,0}^l.
\end{split}
\end{align}
In \eqref{eq:C_diffs}, $\frac{\partial C}{\partial \theta_j^l}$  are expressed as a difference of two terms in the expected cost function with specific network structures switched on or off and therefore measure the usefulness of these structures.

 The gradients $\frac{\partial \hat C}{\partial  \theta_j^l}, \,\,j\in\{1,2,B\}$ can be approximated via forward-propagations through the network using samples $\hat \Xi$ as shown in \cite{GUENTER2024_robust} for unit pruning and \cite{guenter_concurrent} for layer pruning.
Such approximations yield unbiased estimates and are of importance for the precise stochastic approximation results and theorems in these works.
However, in networks with many layers and especially many units, this estimation becomes computationally unattractive and as \cite{GUENTER2024_robust, guenter_concurrent} have shown, a first order Taylor approximation of the differences of the expectations $C$ in \eqref{eq:C_diffs} provides a computationally efficient and sufficiently accurate estimator for the proposed pruning algorithms. This approximation is equivalent to the straight-through estimator \cite{bengio_estimating_2013} and is given by

\begin{align*}
\frac{\partial C}{\partial \theta_B^l} \approx  \frac{\partial \hat C}{\partial \hat \xi_B^l}, \quad
\frac{\partial C}{\partial \theta_{1}^l} \approx  \frac{\partial \hat C}{\partial \hat \xi_{1}^l}, \quad
\frac{\partial C}{\partial \theta_{2}^l} \approx \frac{\partial \hat C}{\partial \hat \xi_{2}^l},
\end{align*}
where the gradients of $\hat C$ are computed via mini-batch sampling the data and a sample $\hat \Xi$ via standard backpropagation.
Finally, the derivatives of $ J_F^l$ and $ J_P^l$ with respect to $\theta_j^l$ and $\theta_B^l$ needed in \eqref{eq:Lgrads_theta} are readily obtained from \eqref{eq:cmpx_measure} with \eqref{eq:comp_cmpx} and \eqref{eq:param_cmpx}.

Algorithm~\ref{alg:learning_algo} shows  pseudo-code summarizing the proposed learning and pruning algorithm.
We initialize the process with the dataset $\mathcal{D}=\lbrace (x_i,y_i)\rbrace_{i=1}^N$ and by selecting the starting network in the structure of~ \eqref{eq:struct_res} defined by the number and type of blocks and the width of each layer. We also specify the mini-batch size $B$ and the hyperparameters $\nu > 0, \alpha\geq 0, \beta \in [0,1]$ for our cost-aware pruning regularization and $\lambda>0$ for weight regularization.

During each iteration, the algorithm goes through four steps. In Step~1, we sample a mini-batch sample $\hat X, \hat Y=\lbrace ( x_i, y_i)\rbrace_{i=1}^B$ with replacement from $\mathcal{D}$ and also obtain samples $\hat \Xi \sim Bernoulli(\Theta)$ that are used to estimate the required expectations.
In Step~2, we compute the gradients of the objective \eqref{eq:L_train_obj_final} with respect to the network weights $W$ and $\Theta$ parameters from \eqref{eq:W_grad} and \eqref{eq:Lgrads_theta}, respectively.
In Step~3, learning takes place via stochastic gradient descent using the gradients computed in Step~2 and the weights $W$ and variational parameters $\Theta$ are updated.

In Step~4, we remove the blocks and units with variational parameters that have converged to zero at the end of each epoch during training; by doing so, the cost of further training computations is decreased accordingly.
Theorem~\ref{thm:stab_layer} in Section~\ref{sec:convergence_results} guarantees that the removed structures cannot recover if the weights $W_1^l, W_2^l$ and update rate $\theta_B^l$ of a nonlinear layer path enter and remain in the region $\mathcal{D}_B$ defined in \eqref{eq:RoA_layer}, since convergence of $\{W_1^l\rightarrow 0,\,W_2^l\rightarrow 0,\, \theta_B^l \rightarrow 0\}$ is then assured for this layer.
If the fan-in and fan-out weights as well as the update rate $\theta_{ji}^l$, $j=1$ or $2$ of the $i$th unit of a layer enter and remain in a region $\mathcal{D}_U$ defined in \eqref{eq:RoA_unit}, the unit can be pruned based on Theorem~\ref{thm:stab_unit} in Section \ref{sec:convergence_results}.
However, a more practical pruning condition, justified by these results and which has been found to work well for unit-only \cite{GUENTER2024_robust} or layer-only pruning \cite{guenter_concurrent} is to prune a structure (block or unit) once its $\theta\leq\theta_{tol}$, a user-defined parameter.
Finally, we check for convergence and in the opposite case, we set $n=n+1$ and go through the previous steps.

Although, we develop in detail our algorithm for fully-connected networks, convolutional networks can be trivially handled by introducing Bernoulli random variables $\xi_B^l$ for the collection of output maps of the layer with weights $W_2^l$ and $\xi_1^l, \xi_2^l$ for the individual output maps of the layer with $W_1^l$ and $W_2^l$, respectively. Indeed in Section~\ref{sec:experiments}, the algorithm is also used on convolutional neural networks.

\begin{algorithm}
	\caption{Learning/Pruning Algorithm}
		\label{alg:learning_algo}
		\begin{algorithmic}[.]
			\REQUIRE  Dataset $\mathcal{D}=\lbrace (x_i,y_i)\rbrace_{i=1}^N$ and mini-batch size $B$, Initial Network Structure with Initial Weights $W(0), V(0)$ and parameters $\Theta(0)$;
			hyperparameters:  $\nu>0, \alpha \geq 0, \beta\in [0,1]$, $\lambda>0$, $\theta_{tol}>0$. \vspace{-0.3cm}
			\\ \hrulefill \\
			\STATE set $n=0$, training$=$True
			\WHILE{training}
			\STATE\vspace{-0.25cm}
			\hrulefill \\
			\textbf{STEP 1: Forward Pass}\\
			\STATE Predict the network's outputs $\hat y_i$ using \eqref{eq:struct_res} and samples $\hat X, \hat Y=\lbrace ( x_i, y_i)\rbrace_{i=1}^B$ and $\hat \Xi \sim Bernoulli(\Theta)$.\\
			\hrulefill \\
			\textbf{STEP 2: Backpropagation Phase}\\
			\STATE Approximate the gradients with respect to weights $W$ using \eqref{eq:W_grad} and $\Theta$ using \eqref{eq:Lgrads_theta}\\
			\hrulefill \\
			\textbf{STEP 3: Learning Phase}\\
			\STATE Take gradient steps to obtain $W(n+1)$ and $\Theta(n+1)$ according to \eqref{eq:DE_W} and \eqref{eq:DE}, respectively.\\
			\hrulefill \\
			\textbf{STEP 4: Network Pruning Phase}\\
			\STATE Once every epoch and for all $\theta$ in $\Theta(n+1)$: if $\theta \leq \theta_{tol}$, prune the corresponding structure from the network by setting its weights to zero and removing it from the network.\\	
			\hrulefill \\
			\STATE set $n=n+1$ and determine whether training$=$False
			\ENDWHILE
		\end{algorithmic}
\end{algorithm}

\subsection{Convergence Results and Pruning Conditions}\label{sec:convergence_results}
In this section, we invoke well-established stochastic approximation results to study the convergence properties of the stochastic gradient descent algorithm of Section~\ref{sec:algorithm}.
More specifically, we consider the following projected ODE:
\begin{align}\label{eq:ODE_proj}
\begin{split}
\dot W &= - \frac{\partial L(W,\theta)}{\partial W}\\
\dot \Theta &= -\frac{\partial L(W,\Theta)}{\partial \Theta} - h,
\end{split}
\end{align}
having as states all trainable parameters $W$ and $\Theta$ of the NN and with $L(W,\Theta)$ from \eqref{eq:L_train_obj_final} and the term $h$ accounting for the projection of $\Theta$ to the positive unit hypercube $\mathcal{H}$. Following \cite[Theorem 2.1, p.127]{kushner2003stochastic}, it can be shown under certain conditions (including the selection of the stepsize $a(n)$ to satisfy the Robbins-Monro rules) that any convergent subsequence generated by the stochastic gradient descent algorithm of Section~\ref{sec:algorithm} approaches a solution of the projected ODE \eqref{eq:ODE_proj} and almost surely converges to a stationary set of the ODE in $\mathcal{H}$ and therefore to a stationary point of \eqref{eq:L_train_obj_final}.
Such assumptions have been verified in \cite[Appendix A]{guenter_concurrent}  for layer pruning and can be shown to hold for the  algorithm proposed here, which performs unit and layer pruning at the same time. However, we omit such details and only present the following results about the asymptotic convergence of \eqref{eq:ODE_proj} to solutions corresponding to pruned deterministic networks.

\subsubsection{Domains of Attraction and Pruning}
Next, we analyze the dynamics of the weights and $\theta$ parameters of specific structures (blocks and units) of the network in more detail. More specifically, we derive conditions under which such weights and $\theta$ parameters are assured to converge to zero and the corresponding structures can be safely pruned from the network.
Under the assumption that the network does not collapse completely, i.e., some information is always propagated, we have that $\norm{\theta_2^l}_1 \neq 0$. A consequence of this assumption is that $\log\gamma_2^l<0$ in \eqref{eq:gamma_choices} for $\beta \in [0,1]$. For $\log\gamma_2^l<0$, the stability results of \cite{GUENTER2024_robust} for unit pruning are readily extended to the dynamics of $\theta_{2i}^{l+1}$ and its corresponding fan-in weights in $W_2^l$ and $W_3^l$ as well as its fan-out weights in $W_1^{l+1}$ and $W_3^{l+1}$. As discussed in \cite{GUENTER2024_robust}, this leads to a practical pruning condition for a unit when its $\theta_{2i}^l$ is small.

Therefore, we focus on the dynamics of $\theta_{1i}$, $\theta_B$ parameters and associated weights.
Due to the combination of layer and unit pruning considered here, the convergence results and corresponding pruning conditions derived in \cite{GUENTER2024_robust} for unit pruning and \cite{guenter_concurrent} for layer pruning fail to apply. This is because the effective regularization on $\theta_B^l$ or $\theta_1^l$ vanishes if either $\norm{\theta_1^l}_1=0$ or $\theta_B^l=0$, respectively, as can be seen from \eqref{eq:L_train_obj} and \eqref{eq:gamma_choices} (or from \eqref{eq:L_train_obj_final} with \eqref{eq:cmpx_measure}, \eqref{eq:comp_cmpx}, \eqref{eq:param_cmpx}). In addition, unlike with  \cite{GUENTER2024_robust} and \cite{guenter_concurrent} where $\log\gamma_B,\log\gamma_1,\log\gamma_2$ are hyperparameters chosen explicitly and remain fixed, here they are determined implicitly and vary dynamically as in \eqref{eq:gamma_choices}. Thus, a different and more careful analysis of the dynamics of $\theta_B^l$ and $\theta_1^l$ than the respective ones in \cite{guenter_concurrent}, \cite{GUENTER2024_robust} is necessary.
In the following, we give two stability theorems establishing the attractive property of certain sets inside of which either a block's $\theta_B^l\rightarrow 0$ or a unit's $\theta_{1i}^l\rightarrow 0$. Convergence to such a set guarantees that a structure of the network will be pruned automatically by the training process. In practice, this leads to simple pruning conditions that can be applied during training as outlined in the Algorithm~\ref{alg:learning_algo} of Section \ref{sec:algorithm}.

We focus on a single block $l$ and drop the superscript $l$ from its parameters  $\theta_B^l, \theta_{1i}^l$ as well as weights $W_1^l,W_2^l$ and cost $C_{ji, \hat \xi^l_{ji}}^{l, \hat \xi_B^l}$ for the sake of clarity. Let $w_{1i}$ be the fan-in weights of the unit with parameter $\theta_{1i}$ and let $w_{2i}$ be the fan-out weights of the same unit.
The following dynamical system derived in \ref{app:dyn_sys_derivation} is a sub-system of \eqref{eq:ODE_proj} describing the dynamics of these variables:
\begin{align}\label{eq:dyn_sys}
\begin{split}
\dot w_{1i} &= -\theta_B\theta_{1i} \nabla_{w_{1i}} C_{1i,1}^{1} - \lambda w_{1i} \\
\dot w_{2i} &= -\theta_B \theta_{1i} \nabla_{w_{2i}} C_{1i,1}^{1} -\lambda w_{2i} \\
\dot \theta_{1i} &= - \theta_B\left[C_{1i,1}^{1}-C_{1i,0}^{1} +R\right]  -h_{1i}, \quad \text{for}\quad  i=1,\dots, K,\\
\dot \theta_B & = - \left[C^{1}-C^{0} + \norm{\theta_1}_1 R\right] -\nu\alpha- h_B,
\end{split}
\end{align}
with
\begin{align}\label{eq:R_Rm}
R  \equalhat
\nu \frac{\partial^2 J_{FP} }{\partial \theta_B \partial \theta_{1i}}
= \nu \left(\frac{\beta}{F}  \left(f_1^l\norm{\theta_2^l}_1+f_a^l+f_2^l\norm{\theta_2^{l+1}}_1\right)
+ \frac{1-\beta}{P} \left(p_1^l\norm{\theta_2^l}_1+1+p_2^l\norm{\theta_2^{l+1}}_1\right)\right)
\end{align}
and $h_{1i},\, i=1,\dots, K$ and $h_B$ projection terms used to keep the $\theta$ parameters in the interval $[0,\, 1]$. Formulas for $h_{1i},\, i=1,\dots, K$ and $h_B$ are stated in \ref{app:dyn_sys_derivation}.
We remark that since $\nu > 0$ and $\beta \in [0,1]$, it is
	\begin{align}\label{eq:Rm_neg_def}
		R \geq R_m >0.
	\end{align}
	with, for example, $R_m =\nu   \left(f_a^l  \frac{\beta}{F}  +\frac{(1-\beta)}{P}\right)$.

Furthermore, we assume:
	\begin{itemize}\hypertarget{item:assumptions}
		\item{(A1)}
		 all weights $W$ of the Neural Network \eqref{eq:struct_res} remain bounded, i.e., it holds for all $n$:
		\begin{align*}
		\norm{W^l_j(n)}<\infty, \, j\in\{1,2,3\},\ \ l=1,\ldots,L,
		\end{align*}
		\item{(A2)} all activation functions used in the Neural Network \eqref{eq:struct_res} are continuously differentiable with bounded derivative $a^{l'}(\cdot)$ and satisfy $a^l(0)=0$,
		\item{(A3)} the data satisfies
		\begin{align*}
		\begin{split}
		\E_{x\sim p(x)}\norm{x}^k  < \infty,\,\, k\in\{1,2,3,4\} \quad
		\text{and} \quad \E_{y\sim p(y|x)}\norm{y}^k  < \infty, \,\, k\in\{1,2\},
		\end{split}
		\end{align*}
	\end{itemize}

Next, we state the two stability theorems justifying layer and unit pruning during training.

\begin{theorem}\label{thm:stab_layer}
	Consider the dynamical system \eqref{eq:dyn_sys} with $\lambda>0$, $\nu > 0$, $\alpha\geq 0$, $\beta \in [0,\,1]$.
	Also, let $x_B=\{W_1,W_2,\theta_B\}$ and consider the Lyapunov candidate function $\Lambda_B(x_B) =\frac{1}{2} \left(\norm{W_1}^2 + \norm{W_2}^2 + \theta_B^2 \right)$ and the set
	\begin{align}\label{eq:RoA_layer}
	\mathcal{D}_B = \left\{W_1,W_2, \theta_B \in[0,1]^K, \theta_1 \in[0,1] \mid \Lambda_B(x_B) \leq \frac{1}{2}\frac{R_m^2}{16(\eta+\kappa)^2}\right\}
	\end{align}	
	with $R_m$ satisfying \eqref{eq:Rm_neg_def}.
	Then under assumptions \hyperlink{item:assumptions}{(A1)-(A3)} and furthermore assuming that the dynamics \eqref{eq:dyn_sys} are piecewise continuous in time and locally Lipschitz in $(x_B, \theta_1)$, uniformly in time, the set
	\begin{align*}
	E_B =  \{x_B\in \mathcal{D}_B,\, \theta_1\in [0,1]^K   \mid W_1=0,\, W_2=0,\, \norm{\theta_1}_1\cdot\theta_B=0\}
	\end{align*}
	is an invariant set in $\mathcal{D}_B$ of the system \eqref{eq:dyn_sys} and locally asymptotically stable (attractive).
	Moreover, $\mathcal{D}_B$ belongs to the region of attraction of $E_B$.
\end{theorem}
\begin{proof}
  See \ref{app:proof_layer}.
\end{proof}

Theorem~\ref{thm:stab_layer} implies that once $\mathcal{D}_B$ is entered, the trajectory generated by \eqref{eq:dyn_sys} tends to $W_1=0, W_2=0$ and either $\norm{\theta_1}=0$ and/or $\theta_B=0$. Then, pruning the layer is justified as soon as $\mathcal{D}_B$ is entered. In practice, a more practical condition based on the previous theoretical result was found to work well: if either $\norm{\theta_1}\leq \theta_{tol}$ or $\theta_B\leq \theta_{tol}$ is observed during training for a small threshold $\theta_{tol}$, we prune the layer.

\begin{theorem}\label{thm:stab_unit}
	Consider the dynamical system \eqref{eq:dyn_sys} with $\lambda>0$, $\nu > 0$, $\alpha\geq 0$, $\beta \in [0,\,1]$.  Also,
	let $x_U=\{w_{1i},w_{2i},\theta_{1i}\}$ and consider the Lyapunov candidate function $\Lambda_U(x_U) =\frac{1}{2} \left(\norm{w_{1i}}^2 + \norm{w_{2i}}^2 + \theta_{1i}^2 \right)$ and the set
	\begin{align}\label{eq:RoA_unit}
	\mathcal{D}_U = \left\{w_{1i},w_{2i}, \theta_B \in[0,1], \theta_{1i} \in[0,1] \mid \Lambda_U(x_U) \leq \frac{1}{2}\frac{R_m^2}{16(\eta+\kappa)^2}\right\}
	\end{align}	
	with $R_m$ satisfying \eqref{eq:Rm_neg_def}.
	Then under assumptions \hyperlink{item:assumptions}{(A1)-(A3)} and furthermore assuming that the dynamics \eqref{eq:dyn_sys} are piecewise continuous in time and locally Lipschitz in $(x_U, \theta_B)$, uniformly in time, the set
	\begin{align*}
	E_U =  \{x_U\in \mathcal{D}_U,\, \theta_B\in [0,1] \mid w_{1i}=0,\,w_{2i}=0,\, \theta_{1i}\cdot\theta_B=0\}
	\end{align*}
	is locally asymptotically stable (attractive).
	Moreover, $\mathcal{D}_U$ belongs to the region of attraction of $E_U$.
\end{theorem}
\begin{proof}
  See \ref{app:proof_unit}.
\end{proof}

Theorem~\ref{thm:stab_unit} implies that once $\mathcal{D}_U$ is entered, the trajectory generated by \eqref{eq:dyn_sys} tends to $w_{1i}=0, w_{2i}=0$ and either $\theta_{1i}$ and/or $\theta_B=0$. Then, pruning the $i$th unit and its fan-in and fan-out weights is justified as soon as $\mathcal{D}_U$ is entered. As for layer pruning, a more practical condition based on the previous theoretical result was found to work well: if either $\theta_{1i}\leq\theta_{tol}$ or $\theta_B\leq\theta_{tol}$ is observed during training for a small threshold $\theta_{tol}$, we prune the unit.
In  case when all $\theta_{1i}$ of a layer are verified to converge to zero, but not $\theta_B$, we can still prune the complete block from the network, since all $w_{1i}$ and $w_{2i}$ and therefore $W_1$ and $W_2$ converge to zero as well.

We remark that the assumption that the dynamics \eqref{eq:dyn_sys} are locally Lipschitz in $(x_B, \theta_1)$, uniformly in time in Theorem~\ref{thm:stab_layer} and Theorem~\ref{thm:stab_unit} holds if in addition to the assumptions \hyperlink{item:assumptions}{(A1)-(A3)} it is required that all activation function $a^l(\cdot), h^l(\cdot)$ of the network \eqref{eq:struct_res} are twice continuously differentiable with a bounded second order derivative. Further, since all time varying terms in \eqref{eq:dyn_sys} are functions of variables of other layers and units that are continuous functions of time, the dynamics \eqref{eq:dyn_sys} are continuous as required in Theorem~\ref{thm:stab_layer} and Theorem~\ref{thm:stab_unit}.

\section{Review of Related Literature}\label{sec:related_work}
\noindent\emph{Pruning Pre-Trained Networks:}
Such pruning methods require pre-trained networks that are expensive to obtain when transfer learning is not applicable. Furthermore, they compress the network by iterating between pruning and fine-tuning or knowledge distillation phases which add additional computational load.

In this group and for layer pruning, \textit{Discrimination based Block-level Pruning} (DBP) \cite{DBPwang2019} and also \cite{chen19_shallowing} identify and remove layers from feed-forward and residual DNNs by analyzing the output features of different layers.
Linear classifiers are trained to generate predictions from  features extracted at different intermediate layers and discover the layers that have the least impact on the network's overall performance.
Other approaches aim to prune the residual blocks of least importance from a NN towards achieving a desired pruning ratio. For example, in \cite{LPSR_zhang}
Batch Normalization layers are used and by setting their factors to zero the preceding block can be effectively eliminated.
The impact to the loss function is approximated using first-order gradients to define an importance score for each residual block; the method prunes the blocks with the lowest score.

A unit pruning method requiring pre-trained networks is for example Filter Thresholding (FT) \cite{li_pruning_2017}; FT forms a score based on the $\mathcal{L}_2$-norm of the fan-out weights from a unit or the kernel matrices in fully-connected or convolutional layers, respectively, and eliminates the units with smallest score to achieve a target sparsity.
SoftNet \cite{he_soft_2018} performs under the same principle using the $\mathcal{L}_1$-norm and ThiNet \cite{luo_thinet:_2017} recursively prunes structures, the elimination of which has the smallest effect in the pre-activation of the following layer to achieve a specified pruning level.

The method of \cite{Beckers_pruning} is based on variational methods and considers the change in variational free energy (variational lower bound, ELBO), when shrinking the prior distribution of each weight of a pre-trained Bayesian NN to zero. If the variational free energy decreases under the new prior, the weight is pruned from the network.
\newline
\newline
\noindent\emph{Simultaneous Training and Pruning:}
These methods carry out the pruning simultaneously with the initial network training and hence do not require pre-trained networks; also, they hold the potential for reducing the computational load of training as well as that of inference.
\cite{sparse_huang2018, xu2020layer, Liu2017learning} introduce scaling factors at the output of specific structures such as neurons, group or residual blocks of the network  and place sparsity regularization in form of the $\mathcal{L}_1$-norm on them; as a result, some of these factors are forced to zero during training and the corresponding structure is removed from the network.
\textit{Network slimming} \cite{Liu2017learning}  uses the scaling factors of BN layers as an importance metric to prune units from the network. Again, $\mathcal{L}_1$-norm regularization is placed on the factors and channels are pruned if the factors are small as determined by a global threshold for the whole network.

Another group of NN pruning approaches is founded on Bayesian variational methods using sparsifying priors
\cite{Louizos_2017,molchanov2017variational, Nalisnick_2015, GUENTER2024_robust,pmlr-v97-nalisnick19a, guenter_concurrent}.
\cite{wright2024bmrs} combines the ideas of learning the variational parameters using log-uniform sparsity inducing priors and the iterative evaluation of changes in variational free energy similar to \cite{Beckers_pruning} for unstructured or structured pruning. Other common such priors are Gaussian scale mixtures defined as a zero mean probability distribution function (PDF) with variance specified by a different RV; an important member of this family is the spike-and-slab prior.
The method of Dropout \cite{Srivastava_2014_Dropout} using a Bernoulli distribution can be thought as implementing a spike-and-slab PDF on the neural network weights \cite{gal2015dropout, Louizos_2017}.
\cite{GUENTER2024_robust} and \cite{guenter_concurrent} employ Bernoulli scale RVs for each unit and layer of the network, respectively and learn their distributions concurrently with the network's weights.

To overcome the challenge of avoiding premature pruning of structures due to a possible poor initialization of the weights and the starting architecture, \cite{GUENTER2024_robust} proposes the ``flattening'' hyper-prior for the parameters of the scale RV.
Further,  it is shown in \cite{guenter_concurrent} that the same hyper-prior leads to an optimization problem with solutions at which all Bernoulli parameters are either $0$ or $1$ in each layer of the network and hence training the NN to obtain such a solution yields an optimally pruned deterministic network.
\newline
\newline
\noindent\emph{Combined Layer and Unit Pruning:}
Existing methods that perform layer and unit pruning concurrently are less available. \cite{accelerate-wang21e} combines the layer pruning method DBP of \cite{DBPwang2019} with the unit-wise \textit{network slimming} method of \cite{Liu2017learning} into an iterative pruning and fine-tuning method that requires a pre-trained NN as input. \cite{accelerate-wang21e} makes the modification to \cite{Liu2017learning} that the pruning threshold for units is set locally in each layer rather than globally and uses  down-sampled input images, effectively modifying the dataset, to further improve the networks efficiency.
\cite{Lin2019TowardsOS} jointly prunes layers and units starting from a pre-trained network and by introducing soft masks to scale the outputs of different structures of the NN. \cite{Lin2019TowardsOS} employs an objective function that also promotes sparsity regularization and aims to align the output of the baseline model with the one of the masked network; the resulting optimization problem is solved via generative adversarial learning techniques  whereby the scaling factors in the soft masks are forced to zero to effect pruning. The methods of \cite{accelerate-wang21e,Lin2019TowardsOS} do not directly address the specific challenges of combined layer and unit pruning brought up earlier in this paper and which motivated the proposed algorithm of Section~\ref{sec:algorithm}.

\section{Simulation Experiments}\label{sec:experiments}
We present simulations results using the proposed algorithm on the ImageNet dataset \cite{imagenet_deng2009}, the CIFAR-100 dataset and the CIFAR-10 dataset \cite{cifar10}. Starting from commonly used deep residual networks such as the ResNet50, ResNet56, and ResNet110 networks, we aim to determine the appropriate number of layers of the network as well as the number of units in each layer concurrently with learning the network weights during training. The goal is to obtain significantly smaller networks with minimally degraded test accuracy when compared to the baseline performance of the trained larger networks.
We fit ResNet50 in our framework by taking $M=2$ in the structure of Figure~\ref{eq:struct_res}, the skip connection weights $W_3^l$ to be either fixed identity or trainable as needed to adjust the dimensionality
and the activation functions $a(\cdot)^l$ and $h^l(\cdot)$ to be ReLU; also, we do not consider any $\xi_2^l$ RVs, i.e., we fix $\theta_{2i}^l=1$. Nevertheless, all $W_1$ and $W_2$ layers within a block have the potential to be reduced by the $\xi_1$ RVs either from the column- or row-side. Resnet56 and ResNet110 are obtained similarly from the structure in Figure~\ref{eq:struct_res} by taking $M=1$.

In each of the following experiments, the networks found after training is completed are evaluated on a separate test dataset; cross-validation to select a network during training is meaningless here since  the network structure keeps evolving.
The resulting networks are evaluated in terms of their accuracy, structure and parameter as well as floating point operations (FLOPS) pruning ratios. We define the parameter pruning ratio (pPR) as the ratio of pruned weights to the number of total weights in the initial network architecture. Likewise, the FLOPS pruning ratio (fPR) is defined as the percentage of FLOPS saved by the network found by the algorithm.
We compare our combined unit and layer pruning method to the methods of \cite{accelerate-wang21e} and \cite{Lin2019TowardsOS}. We note that both of these methods require pre-trained networks and carry out the pruning after the training; therefore, they can not save any computational load during training in contrast to our method. We also compare our method with the unit-only pruning \cite{GUENTER2024_robust} and layer-only pruning \cite{guenter_concurrent} algorithms that have shown to yield competitive results.

Table~\ref{tab:hyperparameters} summarizes common training parameters used for our experiments. Other parameters, such as $\nu, \alpha$ and $\beta$, that are specific to our method are provided in the next subsections for each experiment. All simulations are done in Python using several Pytorch libraries. Every $\xi$ multiplication in \eqref{eq:struct_res} and Figure~\ref{fig:NN_struct_specific} is implemented as a Pytorch \textit{nn.Module} class that can be easily inserted into ResNet implementations adding little extra computation. All experiments were run on a desktop computer using a NVIDIA RTX 4090 graphics processing unit (GPU) and Intel Core i9-13900K central processing unit (CPU).

\begin{table}[t!]
	\caption{Summary of the hyperparameters used for training different ResNet architectures on the ImageNet, CIFAR-10 and CIFAR-100 dataset.} \label{tab:hyperparameters}
	\centering
	\begin{tabular}{cccccc}
		& \makecell{ImageNet\\ResNet50} & \makecell{CIFAR-10\\ResNet56/ResNet110 } &\makecell{CIFAR-100\\ ResNet110 }\\
		\hline
		epochs & $100$&  $182$ & $182$ \Tstrut\\
		batch size $B$ & $180$ & $128$ & $128$\Tstrut\\
		weight regularization $\lambda$ & $\expnumber{1}{-4}$ & $\expnumber{6}{-4}$ & $\expnumber{7.5}{-4}$ \Tstrut\\
		\hline		
		$W$-optimizer&  \makecell{SGD+mom(0.9)}  & \makecell{SGD+mom(0.9)} & \makecell{SGD+mom(0.9)} \Tstrut\\
		learning rate & $0.07$ & $0.1$ & $0.1$\Tstrut\\
		\makecell{learning rate schedule} & \makecell{piecewise constant\\$0.1$@$\{30,60,90\}$}&cosine decay to $0$ & \makecell{cosine decay to $0$\\ with restarts @$\{105,150\}$ }\Tstrut\\	
		\hline
		$\Theta$-optimizer& Adam& Adam & Adam\Tstrut\\
		learning rate & $\expnumber{1}{-3}$ & $\expnumber{2}{-3}$& $\expnumber{5}{-3}$\Tstrut\\
		\makecell{learning rate schedule} & constant & constant &constant\Tstrut\\
		initial $\theta$ & $1.0$&$0.75$&$0.75$\Tstrut\\
		pruning threshold $\theta_{tol}$ & 0.1 & 0.1 & 0.1\\
	\end{tabular}
	\centering

\end{table}

\subsection{ImageNet}
We evaluate our algorithm on ResNet50 having about 25 million parameters and requiring about 4 billion FLOPS per presentation on the ImageNet dataset. Table~\ref{tab:ImageNet_ResNet} summarizes the results for different choices of $\nu, \alpha$ and $\beta$. The parameters are chosen to highlight their effectiveness  in trading-off FLOPS vs. parameter pruning and unit vs. layer pruning.
Generally, increasing $\nu$ leads to higher fPR and pPR as both the computational and parameter complexity are penalized more. For the same $\nu= \expnumber{2.4}{-1}$ and $\alpha=0.2$, choosing $\beta=0$ leads to more aggressive parameter pruning and therefore a higher pPR compared to $\beta=1$. We observe that using a small $\alpha$ does not necessarily lead to any layer pruning. However, increasing $\alpha$ to $0.5$ emphasizes layer pruning more and yields a network with only 38 active layers.
Compared to the layer pruning of \cite{guenter_concurrent}, our method with $\nu=\expnumber{2.4}{-1}, \beta=0.5, \alpha=0.5$ achieves virtually the same test accuracy but with a much higher fPR of over $48\%$ and also a higher pPR of $25\%$ while keeping only one extra layer in the network (38 vs 37 from 50).
In \cite{GUENTER2024_robust}, the best result was achieved by a tedious hand-tuned process for choosing the $\log\gamma_1^l$ parameters differently for each layer in a cost function similar to \eqref{eq:L_train_obj}. In this work, in addition to combining layer and unit pruning, these parameters are automatically selected and our method with a much simpler choice of $\nu=\expnumber{4}{-1}, \beta=0.25, \alpha=0$ outperforms \cite{GUENTER2024_robust}-B in all reported metrics. This network achieves also a fPR of $50\%$, similar to \cite{accelerate-wang21e}, albeit with lower test accuracy. We note that in addition to requiring a pre-trained network, \cite{accelerate-wang21e} not only prunes units and layers but also effectively modifies the dataset by changing the size of the input images to the network, to which perhaps the high test accuracy achieved is owed.
The choice $\nu=\expnumber{3.2}{-1}, \beta=0.5, \alpha=0.3$ in our method improves on the result of \cite{Lin2019TowardsOS} by almost $2\%$ in test accuracy with similar pruning ratios.

\begin{table}[t!]
	\caption{Resulting test accuracy, the number of layers in the network, the FLOPS pruning ratio (fPR) and parameter pruning (pPR) ratio of the pruned networks found with our method and \cite{GUENTER2024_robust, guenter_concurrent, accelerate-wang21e, Lin2019TowardsOS} applied to ResNet50 on the ImageNet dataset.} \label{tab:ImageNet_ResNet}
	\centering
	\begin{tabular}{ccccccccc}
		Method & Test Acc. [\%] & Baseline [\%] &Layers Left & fPR[\%] & pPR[\%]\\
		\hline
		
		$\nu=\expnumber{2.4}{-1}, \beta=0.0, \alpha=0.2$ & 75.75  & 75.93 & 50 &  25.20 & 29.28 \\
		
		$\nu=\expnumber{2.4}{-1}, \beta=1.0, \alpha=0.2$ & 75.02 & 75.93 & 47  & 49.46 & 19.37 \\
		
		$\nu=\expnumber{4}{-1}, \beta=0.25, \alpha=0$ & 74.83  & 75.93 & 50 & 50.27 & 41.47 \\
		
		$\nu=\expnumber{2.4}{-1}, \beta=0.5, \alpha=0.5$ & 74.80 & 75.93 & 38  & 48.46 & 25.33 \\
		
		$\nu=\expnumber{3.2}{-1}, \beta=0.5, \alpha=0.3$ & 73.72 & 75.93 & 38  & 55.21 & 32.20 \\
		
		\vspace{0.1cm}\\
		\hline
		Layer Pruning \cite{guenter_concurrent}	& $74.82$  & $76.04$ & 37 &  26.32 & 21.80 \\
		Unit Pruning \cite{GUENTER2024_robust}-A, 	& $73.53$  & $76.04$ & 50 &  52.81 & 42.13  \\		
		Unit Pruning \cite{GUENTER2024_robust}-B, 	& $74.08$  & $76.04$ & 50 &  44.92 & 37.45  \\		
		
		\vspace{0.1cm}\\
		\hline
		\cite{Lin2019TowardsOS}-0.5 &71.80&76.15 & -& 54.90&24.24\\
		\cite{Lin2019TowardsOS}-1 &69.31&76.15 & -& 72.79&59.95\\
		\cite{accelerate-wang21e}&75.90& -& -&53&50
		
	\end{tabular}
	\centering
	
\end{table}

\begin{figure}[t!]
	\centering
	\includegraphics[width=0.44\textwidth]{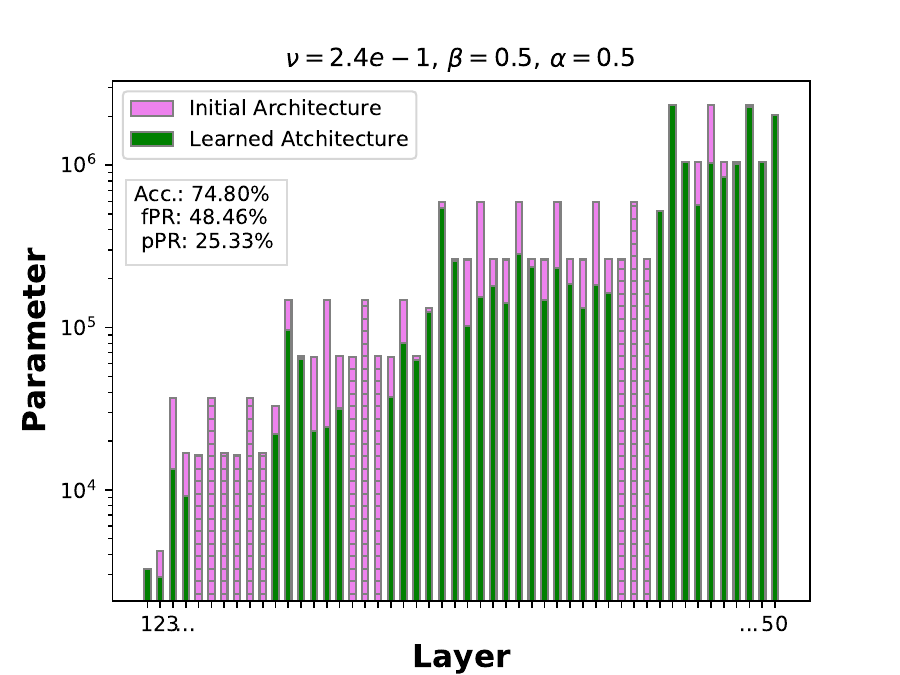}	
	\includegraphics[width=0.44\textwidth]{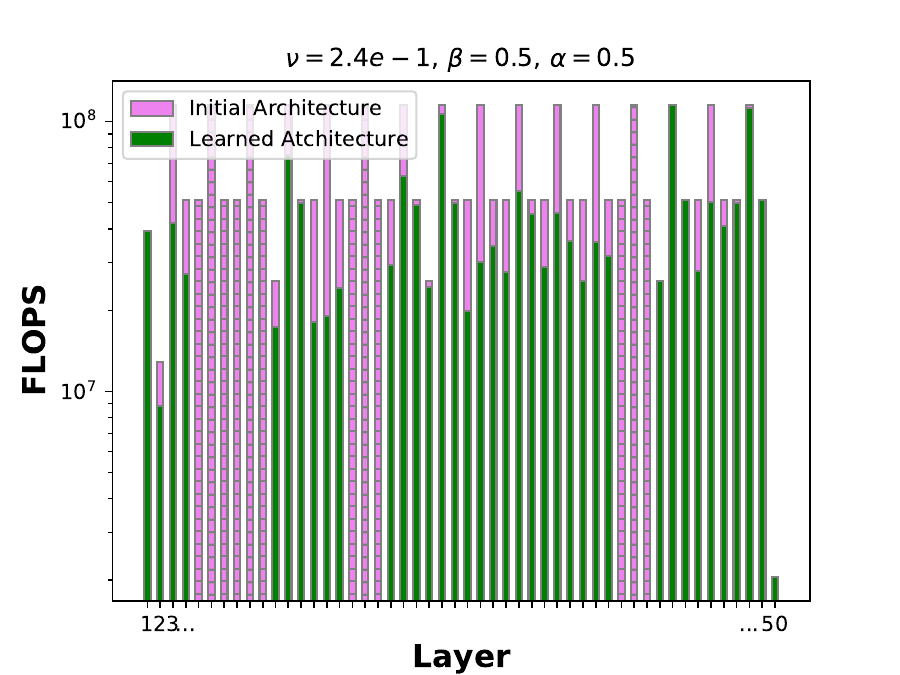}	
	\caption{The number of Parameters and FLOPS in all layers of the learned architecture using our method with hyperparameters $\nu=\expnumber{2.4}{-1}, \beta=0.5, \alpha=0.5$ (green) and of the original ResNet50 architecture (violet) in logarithmic vertical scale. Layers that have been pruned completely are shaded.}
	\label{fig:ImageNet_param_flops}
\end{figure}

Figure~\ref{fig:ImageNet_param_flops} shows the number of parameters and FLOPS in all layers of the baseline network (violet) and the network found using $\nu=\expnumber{2.4}{-1}, \beta=0.5, \alpha=0.5$ (green) in logarithmic scale. If a layer is pruned completely, e.g., because the block's $\theta_B$ is zero, the corresponding bar of the baseline network is shaded. 
Most parameters of ResNet50 are contained in the latter layers of the baseline network, while the computational load is almost uniformly distributed through the layers of the network.

Figure~\ref{fig:ImageNetarchs} shows the number of remaining units on layers where a $\xi_1$ RV is placed and unit pruning occurs, for three hyperparameter choices and for the result in \cite{GUENTER2024_robust}-B. Note that in ResNet50 there are three convolutions per block, but since we placed only two $\xi_1$ RVs (M=2 in Figure~\ref{fig:NN_struct_specific}), not all layers receive such RVs and only the layers that do are shown in this plot.
In the top left plot, $\beta=0$ emphasizes parameter pruning and the network is pruned where most of its parameters are located. The top right plot shows the same network, emphasizing layer pruning, as the one in Figure~\ref{fig:ImageNet_param_flops}.  The lower left plot uses a more balanced $\beta=0.25$ and increases the general pruning level with a higher $\nu$. As a result the front of the network gets pruned as well. In all three cases, our method keeps many units in the layers where the number of unit is increased from its preceding layer. In addition, most of the times the second layer in each block of the network is kept wider than the first  one. This behavior is in contrast to the network found in \cite{GUENTER2024_robust}-B (bottom right) where especially the first two layers with 512 units are pruned heavily, leading to a worse performing network.

\begin{figure}[t!]
	\centering
	\includegraphics[width=0.45\textwidth]{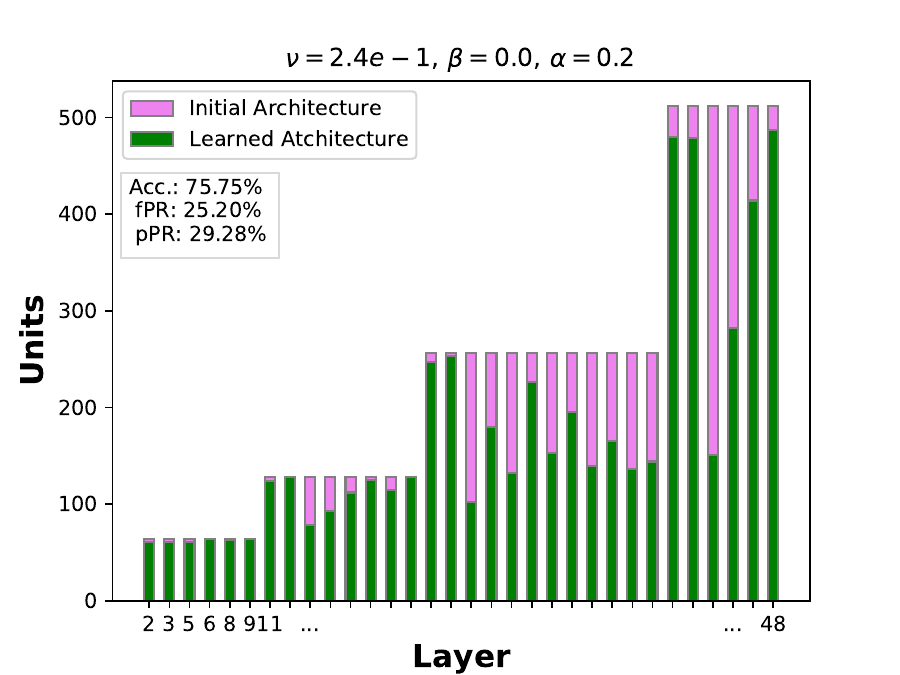}	
	\includegraphics[width=0.45\textwidth]{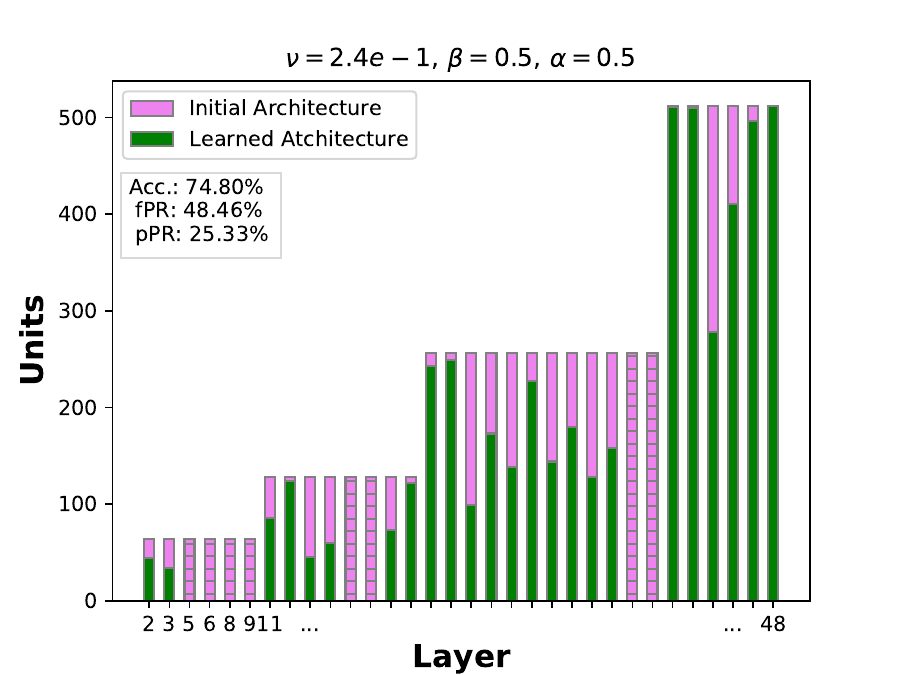}	
	\includegraphics[width=0.45\textwidth]{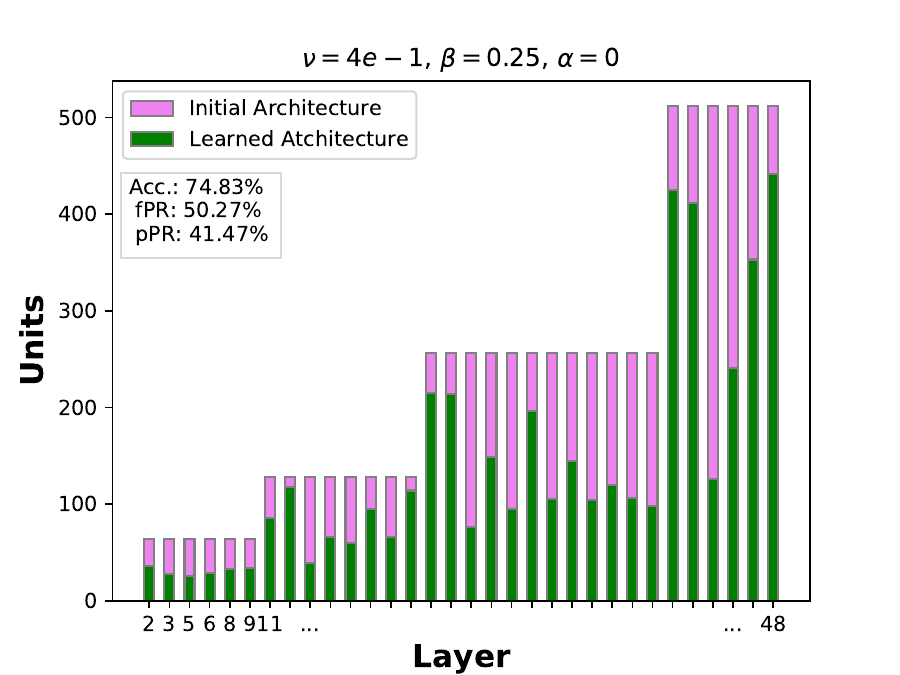}	
	\includegraphics[width=0.45\textwidth]{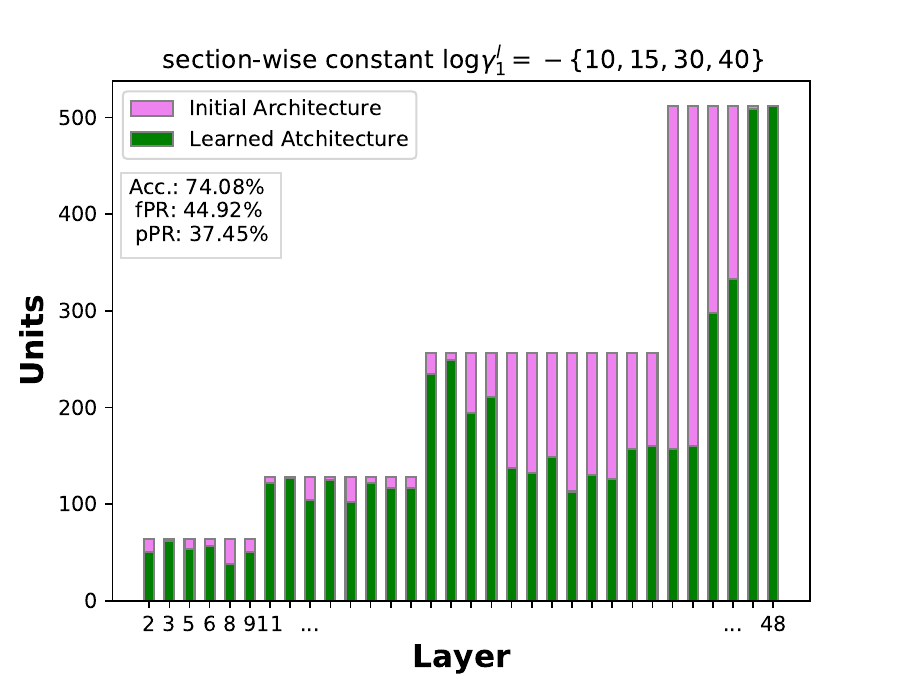}	
	\caption{Top row and bottom left: The number output units of the layers that are multiplied by a $\xi$ RV using our method for different selections of hyperparameters (green) and of the original ResNet50 architecture (violet). Layers that have been pruned completely are shaded. Bottom right: The number output units of the layers of the network found using \cite{GUENTER2024_robust}-B.}
	\label{fig:ImageNetarchs}
\end{figure}

\newpage
Figure~\ref{fig:ImageAccuracy} shows the training (dashed) and test (solid) accuracy of our method for different choices of hyperparameters and the baseline network during the last 50 epochs of training. At epoch 80 all $\Theta$ parameters that have not yet converged to either $0$ or $1$ are rounded to the nearest integer and fixed for the remainder of training. As a result, the training accuracy of the experiments with our method jumps at this epoch.

\begin{figure}[t!]
	\centering
	\includegraphics[width=0.49\textwidth]{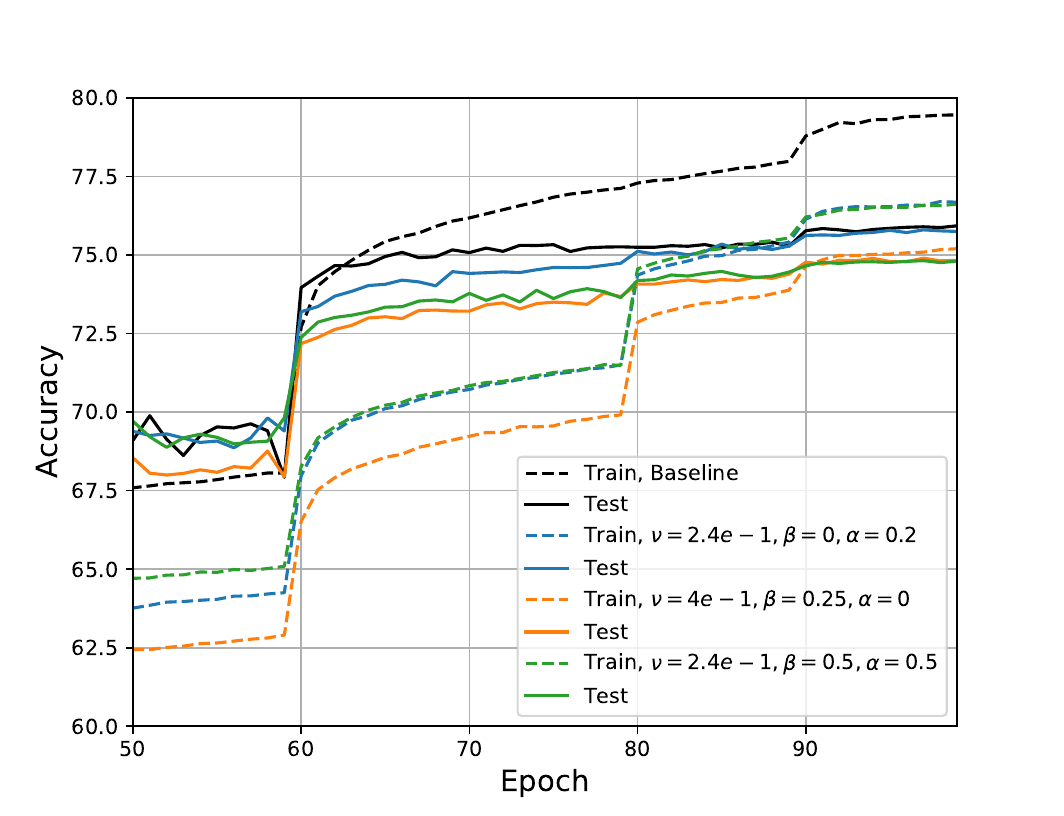}	
	\caption{Test (solid) and train (dashed) accuracy for the ImageNet experiments during the last 50 epochs of training and pruning the ResNet50 architecture with our method and of training the baseline network.}
	\label{fig:ImageAccuracy}
\end{figure}

Table~\ref{tab:final_prediction_time} summarizes the average inference time in milliseconds (after training) over $1000$ batches of size $1$, $16$ and $128$ of various found final architectures with our method and compares them to the baseline ResNet50 model. Using a NVIDIA MX150, a low-power GPU, our found network achieving a test accuracy of $74.80\%$ and reducing the FLOPS of the baseline network by $48.46\%$ ($\nu=\expnumber{2.4}{-1}, \beta=0.5, \alpha=0.5$) decreases the inference time by about $37.6\%$ when using a batch size of $1$. With a low-power CPU such as the Intel Core i5-8250u the inference time is reduced by $35.6\%$. Larger batch sizes of $128$ lead to a decrease of $32.2\%$ in case of the RTX 4090.

\begin{table}[t!]
	\caption{Average inference time over $1000$ batches in milliseconds of various networks found with our method on the ImageNet dataset and ResNet50 as the baseline network. We report times for an NVIDIA RTX 4090 GPU, NVIDIA MX 150 GPU and an Intel Core i5-8250u CPU and for batch sizes of $1$, $16$ and $128$.} \label{tab:final_prediction_time}
	\centering
	\begin{tabular}{c|c|c|c|c|c|c|c|c}
		&&Params(x $10^6$)/&  \multicolumn{3}{c|}{RTX 4090} & \multicolumn{2}{c|}{MX 150} & i5-8250u \\
		&  Layers & FLOPS(x $10^9$) &1 & 16 & 128 & 1 & 16 & 1\\
		\hline
		
		Baseline (ResNet50) &50& 25.56/3.72 &2.03&6.20& 59.25&26.69&234.15 & 149.3 \\
		
		$\nu=\expnumber{2.4}{-1}, \beta=0.0, \alpha=0.2$ & 50& 18.07/2.78&2.18 & 6.42 & 57.81 &22.82 &218.46 & 125.3 \\

		$\nu=\expnumber{4}{-1}, \beta=0.25, \alpha=0$ &50& 14.96/1.85 &2.05  & 5.87 & 51.48 &19.26 & 188.85 & 104.2\\
		
		$\nu=\expnumber{2.4}{-1}, \beta=0.5, \alpha=0.5$ &38& 19.08/1.92 &1.89 & 4.86& 40.19 &16.66 &153.97 & 96.2\\
		
	\end{tabular}
	\centering
	
\end{table}

Figure~\ref{fig:imagenet_timing} depicts the average batch run time with batch size 180 samples during training with our method and the baseline network on the left and the normalized cumulative training time on the right. Our method adds a time-overhead of about $17.9\%$ during the first couple epochs of training when no structures have been pruned and afterwards it successively reduces batch execution time. As a result, most but the least pruned networks take less time to train than the baseline network. For example, with parameters $\nu=\expnumber{2.4}{-1}, \beta=0.5, \alpha=0.5$ achieving a fPR of $48.46\%$, our pruning algorithm saves about $32.6\%$ in cumulative training time when compared to training the baseline network; on the other hand using  $\nu=\expnumber{2.4}{-1}, \beta=0.0, \alpha=0.2$ attaining a fPR of $25.20\%$ increases total training time by only about $10\%$, while maintaining a high test accuracy of $75.75\%$ and achieving considerable inference-time savings when low-power hardware is used (see Table~\ref{tab:final_prediction_time}). We also observe that among obtained networks with a similar fPR, the ones consisting of fewer layers train much faster; this highlights the importance of layer pruning and the introduction of the parameter $\alpha$ and its associated term in \eqref{eq:cmpx_measure} to take advantage of available parallelizable computing hardware. We remark that although the results shown in Figure~\ref{fig:imagenet_timing} depend to a great extent on the computing hardware and the training framework/software, they provide a valid indication of the real-world computational overhead and bottom line savings during training that should be expected from our method. Furthermore, we implemented all multiplications with the $\xi$ variables as individual operation via a Pytorch \textit{nn.Module} class sequentially after each convolutional layer. Unifying the convolutions and $\xi$ multiplications at a lower level could potentially reduce GPU underutilization and the overhead of our method.
Also, it is envisioned that additional savings can be achieved by allowing for dynamically adjusted and increasing batch sizes during training as pruning the network can release a significant amount of GPU memory, not currently exploited in our implementation.

The computational load during training can be also calculated theoretically at any iteration using \eqref{eq:comp_cmpx} and setting all $\theta$ variables that satisfy $\theta>\theta_{tol}$ to $1$, i.e., considering the part of the network that has survived by this iteration (see Algorithm~\ref{alg:learning_algo}). This analysis predicts that networks with a fPR of around $50\%$ save about $40\%$ of the computational cost of the baseline network. However, this saving does not account for the overhead introduced by our method and which was discussed above for the experiments considered.

\begin{figure}[t!]
	\centering
	\includegraphics[width=0.49\textwidth]{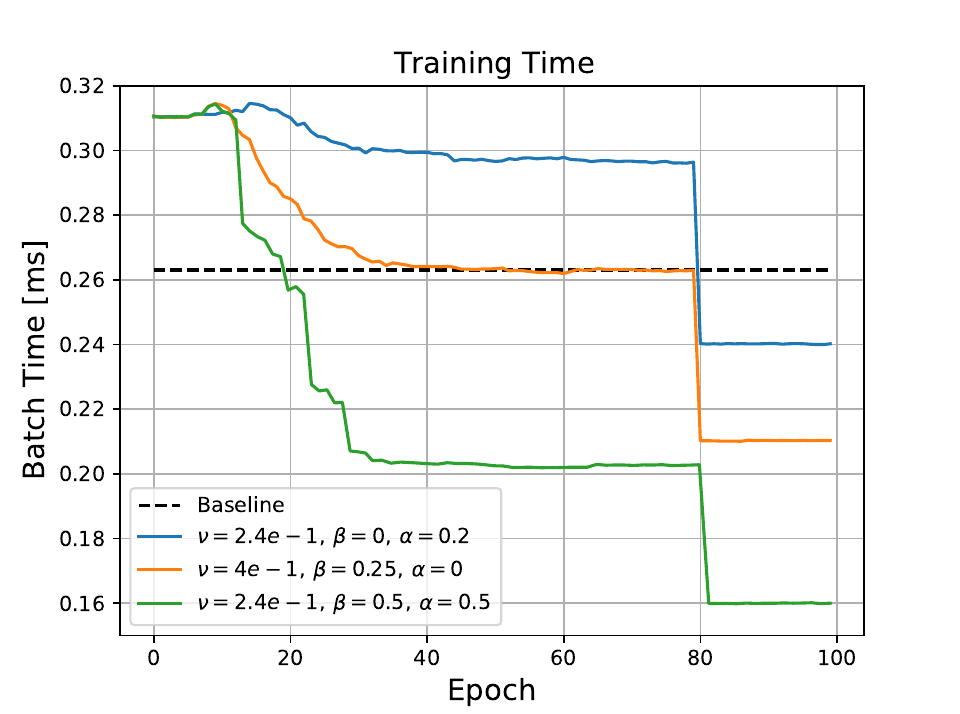}	
	\includegraphics[width=0.49\textwidth]{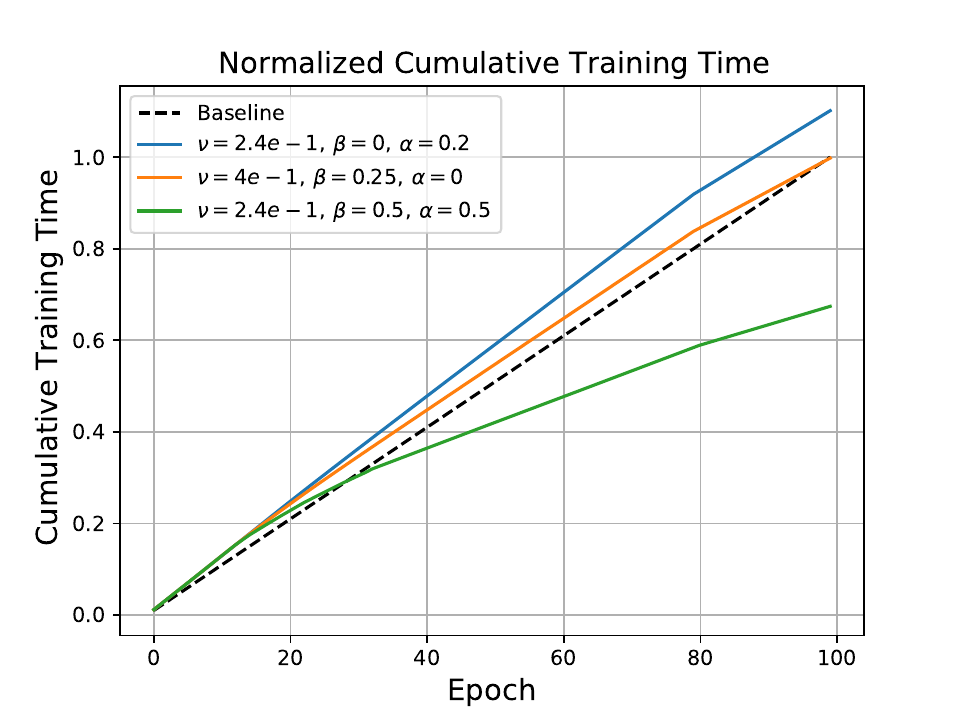}	
	\caption{Left: Average training time per batch during 100 epochs of training ResNet50 with our method on the ImageNet data set and batch size 180 samples. Right: Normalized cumulative training time of our method compared to training the baseline network.}
	\label{fig:imagenet_timing}
\end{figure}

\subsection{CIFAR-100}
We report the results of evaluating our method on the CIFAR-100 dataset using ResNet110 in Table~\ref{tab:CIFAR100_ResNet}. Similar computation (fPR) and parameter (pRP) pruning ratios were achieved in \cite{guenter_concurrent} via layer pruning only after a tedious hand-tuned selection of regularization parameters alike to $\log\gamma_B^l$ in \eqref{eq:L_train_obj} individually for each network block.
In particular, to achieve a higher pPR \cite{guenter_concurrent}-B uses a $5$ times higher regularization parameter in the deeper layers of ResNet110, which hold the most parameters.
The method proposed in this work does not require layer-wise hyperparameters and by choosing $\nu=\expnumber{5.4}{-4}, \beta=0.5,\alpha=0$, the fPR considerably improves over \cite{guenter_concurrent}-B while achieving a higher test accuracy. At the same time the pPR is drastically higher than the one in \cite{guenter_concurrent}-A for a similar accuracy.
We can emulate the layer-only pruning method of \cite{guenter_concurrent} and achieve a fPR of about $46\%$ by taking in \eqref{eq:L_train_obj} $\log\gamma_1^l= \log\gamma_2^l=0$ and $\log\gamma_B^l = -250,\, \forall l=1,\dots, L$.
Using the flexibility of combined layer and unit pruning of the algorithm proposed here, choosing $\nu=\expnumber{1.6}{-1}, \beta=0.5,\alpha=0.1$ improves the fPR to over $48\%$ and also achieves a slightly higher test accuracy.
It is worth noting that from the networks found, the ones with better pruning ratios have more layers than those found by the layer-only pruning method of \cite{guenter_concurrent}; this demonstrates that the proposed method has the potential to discover more balanced structures with better computational properties.

\begin{table}[h!]
	\caption{Resulting test accuracy, the number of layers in the network, the FLOPS pruning ratio (fPR) and parameter pruning (pPR) ratio of the pruned networks found with our method and \cite{guenter_concurrent} applied to ResNet110 on the CIFAR-100 dataset.} \label{tab:CIFAR100_ResNet}
	\centering
	\begin{tabular}{ccccccccc}
		Method & Test Acc. [\%] & Baseline [\%] &Layers Left & fPR[\%] & pPR[\%]\\
		\hline
		
		$\nu=\expnumber{5.4}{-4}, \beta=0.5,\alpha=0$ & 75.20&75.50&102&28.08&22.87\\		
		$\nu=\expnumber{1.6}{-1}, \beta=0.5,\alpha=0.1$ &74.57&75.50& 82 & 48.49 & 44.55\\
		
		\vspace{0.02cm}\\
		\hline
		Layer Pruning \cite{guenter_concurrent}-A	& $75.34$  & $75.50$ & $80$  & $27.99$ & $8.06$ \\
		Layer Pruning \cite{guenter_concurrent}-B	& $75.05$  & $75.50$ & $88$ & $20.06$ & $19.56$ \\			
		\hdashline
        Layer-only pruning by taking&&&&&\\	
		$\log\gamma_1^l= \log\gamma_2^l=0$, $\log\gamma_B^l = -250$& $74.32$  & $75.50$ & $60$ & $46.26$ & $45.05$ \\
        $\forall l=1,\dots, L$\ in \eqref{eq:L_train_obj}&&&&&\\				
		
	\end{tabular}
	\centering

\end{table}

\subsection{CIFAR-10}
We evaluate our method on CIFAR-10 with ResNet56 and ResNet110 to further show its effectiveness and ease of use.
Table \ref{tab:CIFAR10_ResNet56_2} shows the results for ResNet56. We use two different regularization levels of $\nu=\expnumber{8.1}{-2}$ and $\nu=\expnumber{2.2}{-1}$ as well as several choices for $\beta$ to trade-off FLOPS and parameter pruning. Again, computation or parameter pruning can be easily emphasized by varying $\beta$.
Using $\nu= \expnumber{8.1}{-2}, \beta=1.7,\alpha=0$, our method achieves a pPR of 30\%, similar to the layer pruning of \cite{guenter_concurrent}-A, while improving considerably the fPR at the cost of slightly reduced accuracy.
Note that for $\beta=1.7$, we induce a negative term in \eqref{eq:cmpx_measure}. However, due to the specific values of $f_1^l,f_2^l,p_1^l,p_2^l$ in ResNet56 and not considering any $\xi_2^l$ RV, there still exists a constant $R_m$ such that \eqref{eq:Rm_neg_def} holds and the stability results of Section~\ref{sec:convergence_results} remain valid.

\begin{table}[t!]
	\caption{Resulting test accuracy, the number of layers in the network, the FLOPS pruning ratio (fPR) and parameter pruning (pPR) ratio of the pruned networks found with our method and \cite{GUENTER2024_robust, guenter_concurrent, accelerate-wang21e,  Lin2019TowardsOS} applied to ResNet56 on the CIFAR-10 dataset. \label{tab:CIFAR10_ResNet56_2}}	
	\centering
	\begin{tabular}{lclclc}
		Method/Parameter  & Test Acc. [\%]& Bsln. Acc. [\%]&  Layers Left & fPR [\%] & pPR [\%] \\
		\hline	
		$\nu= \expnumber{8.1}{-2}, \beta=1.7,\alpha=0$ &93.40 & 94.26&46 & 48 & 30 \\
		$\nu= \expnumber{8.1}{-2}, \beta=0.5,\alpha=0$ &93.38 &94.26&46 &47 & 50 \\
		$\nu= \expnumber{8.1}{-2}, \beta=0,\alpha=0$ &93.59 &94.26 &50&32 &51 \\
		$\nu= \expnumber{8.1}{-2}, \beta=1, \alpha=0.8$ & 93.07 &94.26&40 & 56 & 55 \\
		\hdashline
		
		$\nu= \expnumber{2.2}{-1}, \beta=1,\alpha=0$& 92.53 &94.26 &34 & 67 & 60 \\
		$\nu= \expnumber{2.2}{-1}, \beta=0,\alpha=0$ & 92.65 & 94.26&42& 51 & 67 \\
		
		\vspace{0.1cm}\\		
		\hline
		Layer Pruning \cite{guenter_concurrent}-A & 93.66 & 94.26& 34 & 39 & 28 \\
		Layer Pruning \cite{guenter_concurrent}-C\ & 92.64 & 94.26&22  & 63 & 46\\
		Unit Pruning \cite{GUENTER2024_robust}  & 92.45 & 94.26&56 & 54 & 84 \\

		\vspace{0.1cm}\\
		\hline
		\cite{accelerate-wang21e}  & 93.76 & 93.69& - & 50 & 40 \\
		\cite{Lin2019TowardsOS}-0.6  & 93.38 & 93.26& 36 & 38 & 12  \\
		\cite{Lin2019TowardsOS}-0.8  & 91.58 & 93.26& 24 & 60 & 66  \\
				
	\end{tabular}
	\centering

\end{table}

The choice $\nu= \expnumber{2.2}{-1}, \beta=1,\alpha=0$ yields a network with similar fPR but much higher pPR when compared to \cite{guenter_concurrent}-C.
In comparison to the unit pruning of \cite{GUENTER2024_robust}, a similar fPR can be achieved using $\nu= \expnumber{2.2}{-1}, \beta=0,\alpha=0$ or when a structure with fewer layers is preferred with $\nu= \expnumber{8.1}{-2}, \beta=1, a=0.8$, which yields a network with only 34 layers but more parameters. We find that unit pruning without combined layer pruning provides the best result in terms of maximal parameter pruning in this experiment.

Our method with the choice $\nu= \expnumber{8.1}{-2}, \beta=0.5,\alpha=0$ leads to the same accuracy as \cite{Lin2019TowardsOS}-0.6 but improves both fPR and pPR considerably, while reducing the computational load of training by a factor of about 1.7. In comparison to the unit and layer pruning of \cite{accelerate-wang21e}, our network achieves a higher pPR at a slightly reduced test accuracy.
We note that \cite{accelerate-wang21e} retrains a pre-trained network requiring additional computational load and perhaps surprisingly yields a pruned network that performs better than its baseline network in this example.

Table~\ref{tab:CIFAR10_ResNet110} summarizes the results using the ResNet110 architecture. In comparison to \cite{Lin2019TowardsOS}, the proposed method achieves higher test accuracies with higher pruning ratios. When compared to \cite{guenter_concurrent}-B, the proposed method with $\nu= \expnumber{2.2}{-1}, \beta=0,\alpha=0$, and hence actively favoring a network with few parameters, achieves a considerably higher pPR for a slight decrease in test accuracy.

\begin{table}[t!]
	\caption{Resulting test accuracy, the number of layers in the network, the FLOPS pruning ratio (fPR) and parameter pruning (pPR) ratio of the pruned networks found with our method and \cite{guenter_concurrent, Lin2019TowardsOS} applied to ResNet110 on the CIFAR-10 dataset.\label{tab:CIFAR10_ResNet110}}	
	\centering
	\begin{tabular}{lclclc}
		Method/Parameter &   Test Acc. [\%]& Bsln. Acc. [\%] & Layers Left & fPR [\%] & pPR [\%] \\
		\hline
		
		$\nu= \expnumber{5.4}{-4}, \beta=0.5,\alpha=0$  &94.55  &94.70 & 90 &41 &44 \\
		$\nu= \expnumber{2.2}{-1}, \beta=1,\alpha=0$ &93.44  &94.70& 56  &77 &74 \\
		$\nu= \expnumber{2.2}{-1}, \beta=0,\alpha=0$  &93.57  &94.70& 65 &64 &78 \\
		
		\vspace{0.1cm}\\
		\hline
		Layer Pruning \cite{guenter_concurrent}-A  &94.25 & 94.70&51  & 54 & 44 \\
		Layer Pruning \cite{guenter_concurrent}-B  & 93.86 & 94.70&41 & 63 & 53 \\
		\vspace{0.1cm}\\
		\hline
		\cite{Lin2019TowardsOS}-0.1 & 93.59 & 93.5&90  & 18.7 & 4.1  \\
		\cite{Lin2019TowardsOS}-0.5  & 92.74 & 93.5&- & 48.5 & 44.8  \\
		
	\end{tabular}
	\centering

\end{table}

\section{Conclusion}\label{sec:conclusions}
We have proposed a novel algorithm for combined unit and layer pruning of DNNs during training that functions without requiring a pre-trained network to apply. Our algorithm optimally trades-off learning accuracy and pruning levels while balancing layer vs. unit pruning and computational vs. parameter complexity using only three user-selected parameters that i) control the general pruning level, ii) regulate layer vs. unit pruning and iii) trade-off computational complexity vs. the number of parameters of the network; these parameters can be chosen to best take advantage of the characteristics of available computing hardware such as the ability for parallel computations. Based on Bayesian variational methods, the algorithm introduces binary Bernoulli RVs taking values either 0 or 1, which are used to scale the output of the units and layers of the network;
then, it solves a stochastic optimization problem over the Bernoulli parameters of the scale RVs and the network weights to find an optimal network structure.
Pruning occurs during training when a variational parameter converges to 0 rendering the corresponding structure permanently inactive, thus saving computations not only at the prediction but also during the training phase. We place the ``flattening'' hyper-prior \cite{GUENTER2024_robust} on the parameters of the Bernoulli random variables and show that the
solutions of the stochastic optimization problem are deterministic networks, i.e., with Bernoulli parameters at either 0 or 1, a property that was first established in \cite{guenter_concurrent} for layer-only pruning and extended to the case of combined unit and layer pruning in this work in Theorem~\ref{thm:deterministic_network}.

A key contribution of our approach is to define a cost function that combines the objectives of prediction accuracy and network pruning in a complexity-aware manner and the automatic selection of the hyperparameters in the flattening hyper-priors. We accomplish this by relating the corresponding regularization terms in the cost function to the expected number of FLOPS and the number of parameters of the network. In this manner, all  regularization parameters, as many as three times the number of layers in the network, are expressed in terms of only three user-defined parameters, which are easy to interpret.  Thus, the tedious process of tuning the regularization parameters individually is circumvented, and furthermore, the latter vary dynamically as needed, adjusting as the network complexity and the usefulness of the substructures they control evolve during training.

We analyze the ODE system that underlies our stochastic optimization algorithm and establish in Theorem~\ref{thm:stab_layer} domains of attraction around zero for the dynamics of the weights of a layer and either its Bernoulli parameter or the Bernoulli parameters of all of its units. Similarly, we establish in Theorem~\ref{thm:stab_unit}, domains of attraction around zero for the dynamics of the weights of a unit and either its Bernoulli parameter of the Bernoulli parameter of its layer. These results provide theoretical support for practical pruning conditions that can be used to eliminate individual units and/or whole layers unable to recover during further training.

Our proposed algorithm is computationally efficient due to the removal of unnecessary structures during training and also because it requires only first order gradients to learn the Bernoulli parameters, adding little extra computation to standard backpropagation. It also requires minimal effort in tuning hyperparameters. We evaluate our method on the CIFAR-10/100 and ImageNet datasets using common ResNet architectures and demonstrate that our combined unit and layer pruning method improves upon layer-only or unit-only pruning with respect to pruning ratios and test accuracy and favorably competes with combined unit and layer pruning algorithms that require a pre-trained network.  In summary, the proposed algorithm is capable of efficiently discovering network structures that optimally trade-off performance vs. computational/parameter complexity.

\appendix
\renewcommand\thesection{\appendixname\ \Alph{section}}
\section{Proof of Theorem~\ref{thm:deterministic_network}}\label{app:proof_deterministic}
\begin{proof}
	We observe that the objective function from \eqref{eq:L_train_obj_final}
	\begin{align*}
		L(W,\Theta) =  C(W,\Theta) + \frac{\lambda}{2}W^\top W+\nu J_{FP}
	\end{align*}
	with is a multilinear function in $\Theta$ as follows.
	It can be easily seen from \eqref{eq:comp_cmpx} and \eqref{eq:param_cmpx} that each $J_F^l$ and each $J_P^l, \, \forall l=1,\dots, L$ is multilinear in $\Theta$. Then, from \eqref{eq:cmpx_measure}, $J_{FP}$ is the sum of these multilinear terms and hence a multilinear function of $\Theta$ itself.
	Moreover, using the fact that all RVs $\Xi$ in the network are Bernoulli distributed, we can write $C(W,\Theta)$ from \eqref{eq:Cost} as
	\begin{align*}
		C(W,\Theta) =&	\frac{1}{N} \sum_I   g_I(\Theta) \log p(Y\mid X,W,\Xi=I)
	\end{align*}
	where
	\begin{align*}
		g_I(\Theta) \equalhat 	\prod_{l=1}^{L} \left(\left({\theta_B^l}\right)^{i_B^l}\left(1-\theta_B^l\right)^{1-i_B^l}  \prod_q \left({\theta_{1q}^l}\right)^{i_{1q}^l}\left(1-\theta_{1q}^l\right)^{1-i_{1q}^l} \prod_q \left({\theta_{2q}^l}\right)^{i_{2q}^l}\left(1-\theta_{2q}^l\right)^{1-i_{2q}^l}\right).
	\end{align*}
	and $I$ denotes a multi-index specified by the values of $i^l_B,i^l_{1q},i^l_{2q}=0\ \text{or}\ 1$, $l=1,\ldots,L$ and $q$ runs over the number of units in each layer.
	Then, $C(W,\Theta)$ is also multilinear in $\Theta$, and we conclude that $L(W,\Theta)$ is a multilinear function of $\Theta$, which is box-constrained, i.e., $0\leq\theta_j^l\leq 1, \, j\in\{1,2,B\}$.
	
	Next, assume that a (local) minimum value of $ L(W,\Theta)$ is achieved for any $0<\theta_B^l<1$ or $0<\theta_{ji}^l<1$.	
	In the case where $0<\theta_B^l<1$ and by fixing all $W$, $\theta_B^k$, $k\neq l$ and $\theta_{ji}^l, \, j\in\{1,2\}$, we obtain a linear function of $\theta_B^l$ from $L(W,\Theta)$. Then, assuming $\frac{\partial L}{\partial\theta_B^l}\neq 0$, $L(W,\Theta)$ can be further (continuously) minimized by moving $\theta_B^l$ towards one of $0$ or $1$, which is a contradiction. If $\frac{\partial L}{\partial\theta_B^l}=0$, $L(W,\Theta)$ does not depend on $\theta_B^l$ and we can replace $\theta_B^l$ with $0$ or $1$, achieving the same minimal value.
	In the case where $0<\theta_{ji}^l<1$ and by fixing all $W$, $\theta_B^l$, $\theta_{h}^l, \, h\neq j$ and $\theta_{jk}^l, \, k\neq i$, we obtain a linear function of $\theta_{ji}^l$ from $L(W,\Theta)$ and repeat the above argument.
\end{proof}

\section{Derivation of the Dynamical System~\eqref{eq:dyn_sys}}\label{app:dyn_sys_derivation}
We focus on the dynamics of $\theta_B^l, \theta_{1i}^l$ as well as weights $W_1^l,W_2^l$ and consider one block $l$ and drop the superscript $l$ from these parameters for the sake of clarity.
	Recall that $C(W,\Theta)$ can be written as in \eqref{eq:C_detail_notation}, exposing its dependence on $\theta_B$ and $\theta_{1i}$.
	Let $w_{1i}$ be the fan-in weights of the unit with parameter $\theta_{1i}$ and let $w_{2i}$ be the fan-out weights of the same unit.
	Then, realizing that $C_{1i,0}^{l,1}$ as well as $C^{l,0}$ are neither functions of the fan-in, nor the fan-out weights,
	the following gradients of $C(W,\Theta)$ are obtained from \eqref{eq:C_detail_notation}:
\begin{align}\label{eq:Cgrads}
	\begin{split}
		\frac{\partial C}{\partial w_{1i}} &=\theta_B\theta_{1i} \nabla_{w_{1i}} C_{1i,1}^{l,1},
		\quad \frac{\partial C}{\partial w_{2i}} = \theta_B \theta_{1i} \nabla_{w_{2i}} C_{1i,1}^{l,1} \\	
		\frac{\partial C}{\partial \theta_{1i}} &= \theta_B (C_{1i,1}^{l,1} - C_{1i,0}^{l,1}),
		\quad \frac{\partial C}{\partial \theta_B} =   C^{l,1} - C^{l,0}.
	\end{split}
\end{align}
The gradients in \eqref{eq:Cgrads} show that $w_{1i}, w_{2i}$ are updated by the product of both $\theta_B$ and $\theta_{1i}$, in contrast to the fan-out weights $w_{3i}^l$ and fan-in weights $w_{3i}^{l-1}$ connected to a unit with $\theta_{2i}$, where only $\theta_{2i}$ would appear instead.

Next we define $h$ in \eqref{eq:ODE_proj}. Let $h$ consist of $h_B$ and $h_{1i}$ such that in each layer:
\begin{align*}
	\dot \theta_B &= -\frac{\partial L(W,\Theta)}{\partial \theta_B}-h_B\\
	\dot \theta_{1i} &= -\frac{\partial L(W,\Theta)}{\partial \theta_{1i}}-h_{1i} \quad \forall i=1,\dots, K.\\
\end{align*}	
The projection term $h_B$ is:
\begin{align}\label{eq:hB_def}
	\begin{split}
		h_B &\equiv \begin{cases}
			-\left[\left[C^{1}-C^{0} +\norm{\theta_{1i}}_1R\right] +\nu\alpha \right]_{+} \quad &\theta_B=0\\
			0 &0<\theta_B<1\\
			\left[\left[C^{1}-C^{0} +\norm{\theta_{1i}}_1R\right] +\nu\alpha \right]_{-} \quad &\theta_B=1
		\end{cases},
	\end{split}
\end{align}
with $\left[\cdot\right]_+$, $\left[\cdot\right]_-$ denoting the positive and negative part of their argument, respectively.
Similarly, the projection term $h_{1i}$ is:
\begin{align}\label{eq:hi_def}
	\begin{split}
		h_{1i} &\equiv \begin{cases}
			-\left[\theta_B\left[C_{1i,1}^{1}-C_{1i,0}^{1} +R\right]  \right]_{+} \quad &\theta_{1i}=0\\
			0 &0<\theta_{1i}<1\\
			\left[\theta_B\left[C_{1i,1}^{1}-C_{1i,0}^{1} +R\right]  \right]_{-} \quad &\theta_{1i}=1
		\end{cases}.\\
	\end{split}
\end{align}

Using the definition of $R$ from \eqref{eq:R_Rm} together with \eqref{eq:Cgrads}, \eqref{eq:hB_def}, \eqref{eq:hi_def}  in \eqref{eq:ODE_proj} yields the dynamical system \eqref{eq:dyn_sys} for each layer.

\newpage
\section{Proof of  Theorem~\ref{thm:stab_layer} and Theorem~\ref{thm:stab_unit}}\label{app:proofs}
We first discuss results needed in the proof of Theorem~\ref{thm:stab_layer} and Theorem~\ref{thm:stab_unit}.
We will make use of the following Lemma:
\begin{lemma}\label{lm:bounds}
	Under assumptions \hyperlink{item:assumptions}{(A1)-(A3)} the following bounds hold:
	\begin{align}\label{eq:bounds}
		\begin{split}
			&(a) \quad\quad | w_{ji}^T \nabla_{w_{ji}} C_{ji,1}^{l,1}| \leq \eta \norm{w_{ji}},  \, j\in\{1,2\}, \\
			\quad \text{and}\quad
			&(b)\quad\quad |C_{1i,1}^{l,1}-C_{1i,0}^{l,1}| \leq \kappa \left( \norm{w_{1i}} + \norm{w_{2i}} \right),\\
		\end{split}
	\end{align}
	where the constants $0 < \eta,\kappa < \infty$ are independent of the remaining weights of the network.
\end{lemma}
Lemma~\ref{lm:bounds} has been proven in \cite[Appendix A]{GUENTER2024_robust}, for non-residual feed-forward NNs
but can be trivially extended to the residual NNs used here, by considering the skip connection as a part of the nonlinear layer path but with an identity activation function.

The following result  shows that  $C^{l,1}-C^{l,0}\rightarrow 0$ as $\theta_1\rightarrow0$ at the same rate.
\begin{lemma}\label{lm:bounds2}
	Under assumptions \hyperlink{item:assumptions}{(A1)-(A3)} the following bound holds:
	\begin{align}\label{eq:delC_bnd}
		|C^{l,1}-C^{l,0}| \leq \sum_i \theta_{1i} \kappa \left(\norm{w_{1i}}+\norm{w_{2i}}\right),
	\end{align}
	where the constants $0 < \eta,\kappa < \infty$ are the same as in Lemma~\ref{lm:bounds} and  are independent of the remaining weights of the network.
\end{lemma}
\begin{proof}
	First, we extend the notation \eqref{eq:C_notation} as follows:
	\begin{align*}
		C_{j\thickbar{k}, 0}^{l, 1} &= \E_{\substack{\Xi\sim  q(\Xi\mid \Theta)}}\left[-\frac{1}{N}\log p(Y\mid X,W, \Xi) \mid \xi_B^l = 1,\, \xi_{j,1}^l = \xi_{j,2}^l=\ldots=\xi_{j,k}^l=0 \right],\quad\text{and}\quad\\
		C_{j\thickbar{k}, 1}^{l, 1} &= \E_{\substack{\Xi\sim  q(\Xi\mid \Theta)}}\left[-\frac{1}{N}\log p(Y\mid X,W, \Xi) \mid \xi_B^l = 1,\, \xi_{j,1}^l = \xi_{j,2}^l=\ldots=\xi_{j,k-1}^l=0,\,\xi_{j,k}^l=1 \right].
	\end{align*}
	
	Then, we readily obtain:
	\begin{align*}
		\begin{split}
			C^{l,1} &= \theta_{11}^l C_{11,1}^{l,1} + \left(1-\theta_{11}^l\right) C_{11,0}^{l,1}
			=\theta_{11}^l \left[C_{11,1}^{l,1}-C_{11,0}^{l,1}\right]+C_{11,0}^{l,1}\\
			C_{11,0}^{l,1} =C_{1 \thickbar 1,0}^{l,1} &= \theta_{12}^l C_{1\thickbar{2},1}^{l,1} + \left(1-\theta_{12}^l\right) C_{1\thickbar{2},0}^{l,1}
			=\theta_{12}^l \left[C_{1\thickbar{2},1}^{l,1}-C_{1\thickbar{2},0}^{l,1}\right]+C_{1\thickbar{2},0}^{l,1}\\
			&\vdots\\
			C_{1\thickbar{K-1},0}^{l,1} &= \theta_{1K}^l C_{1\thickbar{K},1}^{l,1} + \left(1-\theta_{1K}^l\right) C_{1\thickbar{K},0}^{l,1}
			=\theta_{1K}^l \left[C_{1\thickbar{K},1}^{l,1}-C_{1\thickbar{K},0}^{l,1}\right]+C^{l,0}
		\end{split}
	\end{align*}
	where $K$ is the number of units on the non-linear layer and the last relation follows by noting that $C_{1\thickbar{K},0}^{l,1}\equiv C^{l,0}$. Finally, adding the above equations yields
	\begin{align}\label{eq:theta_delC}
		C^{l,1} = \theta_{11}\left[C_{11,1}^{l,1}-C_{11,0}^{l,0}\right]+\ldots+
		\theta_{1K}\left[C_{1\thickbar{K},1}^{l,1}-C_{1\thickbar{K},0}^{l,1}\right]+C^{l,0}
	\end{align}
	and applying (\ref{eq:bounds}b) from Lemma~\ref{lm:bounds} to each of the terms $C_{1\thickbar k,1}^{l,1}-C_{1\thickbar k,0}^{l,0}$ in \eqref{eq:theta_delC} leads to
	\begin{align*}
		|C^{l,1}-C^{l,0}| \leq \sum_i \theta_{1i} \kappa \left(\norm{w_{1i}}+\norm{w_{2i}}\right).
	\end{align*}
\end{proof}

\subsection{Proof of Theorem~\ref{thm:stab_layer}}\label{app:proof_layer}
\begin{proof}	
	First, we compute and bound using \eqref{eq:bounds} (Lemma~\ref{lm:bounds}, \ref{app:proofs}) and \eqref{eq:delC_bnd} (Lemma~\ref{lm:bounds2}, \ref{app:proofs}) the Lie derivative of $\Lambda_B(x_B)$ along \eqref{eq:dyn_sys} as follows:
	\begin{align*}
	\begin{split}
	\dot \Lambda_L &=  \sum_i w_{1i}^T \dot w_{1i} +   \sum_i w_{2i}^T \dot w_{2i}  + \theta_B \dot \theta_B = -\lambda \left(\norm{W_1}^2 +\norm{W_2}^2 \right)\\
	&\hspace*{0.5cm} - \theta_B \sum_i \theta_{1i} w_{1i}^T \nabla_{w_{1i}} C_1^1- \theta_B \sum_i \theta_{1i} w_{1i}^T \nabla_{w_{2i}} C_1^1 - \theta_B \left(C^1-C^0\right) + \theta_B \left(-\nu\alpha-\norm{\theta_1}_1R-h_B\right)\\
	&\leq -\lambda \left(\norm{W_1}^2 +\norm{W_2}^2 \right)
	+ \theta_B \left(\eta\sum_i \theta_{1i} \norm{w_{1i}}  + \eta\sum_i \theta_{1i} \norm{w_{2i}} + |C^1-C^0|-\nu\alpha - \sum_i \theta_{1i} R - h_B\right) \\
	&\leq  -\lambda \left(\norm{W_1}^2 +\norm{W_2}^2 \right) + \theta_B\sum_i \theta_{1i}  \bigg( (\eta+\kappa)(\norm{w_{1i}}+\norm{w_{2i}}) -R \bigg)   -\theta_B\nu\alpha - \theta_B h_B.\\
	\end{split}
	\end{align*}
	In $\mathcal{D}_B$ we have
	\begin{align*}
	\norm{w_{ji}}\leq\norm{W_j} \leq \frac{R_m}{4(\eta+\kappa)}, \,\, j\in\{1,2\},
	\end{align*}
	and obtain that
	\begin{align*}
	\forall i: \, (\eta+\kappa)(\norm{w_{1i}}+\norm{w_{2i}}) -R \leq  \frac{R_m}{2} - R \leq \frac{-R_m}{2},
	\end{align*}
	using $R \geq R_m>0$ from \eqref{eq:Rm_neg_def}.
	Then, the Lie derivative is bounded as follows:
	\begin{equation}\label{eq:Lie_der_B_bnd}
	\dot \Lambda_B \leq -\lambda \left(\norm{W_1}^2 +\norm{W_2}^2 \right)-\theta_B\norm{\theta_1}_1 \frac{R_m}{2} -\theta_B\cdot\nu\alpha - \theta_B h_B 
	\leq \underbrace{-\lambda \left(\norm{W_1}^2 +\norm{W_2}^2 \right)-\theta_B\norm{\theta_1}_1\frac{R_m}{2}}_{\equalhat W_B(x_B, \theta_1)} \leq 0.
	\end{equation}
	In deriving \eqref{eq:Lie_der_B_bnd}, we used $\nu\alpha\geq 0$ and that for $\theta_{B} \in [0,1$], it is $h_B\geq 0$ from \eqref{eq:hB_def} and therefore $\theta_{B}h_{B}\geq 0$.
	If $\alpha>0$, then $\dot \Lambda_B$ is always strictly negative away from $x_B=0$ rendering this point asymptotically stable. However, $f_B=0$ and $\norm{\theta_1}_1=0$ leads to $\dot \Lambda_B = 0$ with possibly $\theta_B\neq 0$.
	In the latter case and under the assumption that the dynamics \eqref{eq:dyn_sys} are piecewise continuous in time and locally Lipschitz in $(x_B, \theta_1)$, uniformly in time, \cite[Theorem 8.4, p323.]{Khalil_nonlinear} applies and states that all solutions of \eqref{eq:dyn_sys} initialized in $\mathcal{D}_B$ are bounded and satisfy $W_B(x_B, \theta_1)\rightarrow 0$. In turn, this implies that $(x_B, \theta_1)$ approaches the set
	\begin{align*}
	&\{x_B \in \mathcal{D}_B, \theta_1\in [0,1]^K \mid W_B(x_B,\theta_1)=0\} \\
	&= \{x_B\in \mathcal{D}_B,\, \theta_1\in [0,1]^K \mid w_{1i}=0,\,w_{2i}=0,\, \norm{\theta_{1}}_1\cdot\theta_B=0\} = E_B.
	\end{align*}
	We note that a set of the form $\{(x_B,\theta_1) \in B_r \mid V_2(x_B,\theta_1) \leq \min_{\norm{(x_B,\theta_1)}=r} V_1(x_B,\theta_1)\}$ with the ball $B_r$, $r>0$ in \cite[Theorem 8.4, p.323]{Khalil_nonlinear} imposes no further restriction on the initial condition in our case, since $V_1(x_B,\theta_1)= V_2(x_B,\theta_1) = \Lambda_B(x_B,\theta_1)$.
	Lastly, we show invariance of the set $E_B$ for \eqref{eq:dyn_sys}. Since in $E_B$ it holds $W_1=0$, $W_2=0$, $\|\theta_1\|_1\cdot\theta_B=0$, it follows from the dynamics \eqref{eq:dyn_sys} that $\dot W_1=0$, $\dot W_2=0$ and thus $W_1$, $W_2$ remain at zero. In turn, this leads to $C^1-C^0=0$. Also, notice that $R>0$ from \eqref{eq:R_Rm} due to $\nu >0$ and $\beta\in [0,\,1]$.
	Then from \eqref{eq:dyn_sys}, assuming $\theta_B=0$ implies $\dot\theta_1=0$, i.e., $\theta_1$ remains constant in $[0,1]^K$ and it also follows $\dot\theta_B=0$ and that $\theta_B$ remains equal to zero.
	On the other hand, assuming $\theta_1=0$ implies $\dot\theta_B=0$ and $\theta_B$ remains constant, which in turn implies that $\dot\theta_1=0$ and $\theta_1$ remains equal to zero. Thus in both cases, we conclude that $E_B$ is an invariant set of \eqref{eq:dyn_sys}.
\end{proof}

\subsection{Proof of Theorem~\ref{thm:stab_unit}}\label{app:proof_unit}
\begin{proof}
	First, we compute and bound using \eqref{eq:bounds} (Lemma~\ref{lm:bounds}, \ref{app:proofs}) the Lie derivative of $\Lambda_U(x_U)$ along \eqref{eq:dyn_sys} as follows:
	\begin{align*}
	\dot \Lambda_U &=   w_{1i}^T \dot w_{1i} +  w_{2i}^T \dot w_{2i}  + \theta_{1i}\dot \theta_{1i}
	= -\lambda \left(\norm{w_{1i}}^2 +\norm{w_{2i}}^2 \right)\\
	&\hspace*{0.5cm} - \theta_B \theta_{1i} w_{1i}^T \nabla_{w_{1i}} C_{1i,1}^{1} - \theta_B\theta_{1i} w_{2i}^T\nabla_{w_{2i}} C_{1i,1}^{1}- \theta_B\theta_{1i} \left(C_{1i,1}^{1}-C_{1i,0}^{1}\right) - \theta_B\theta_{1i} R - \theta_{1i}h_{1i}\\
	&\leq -\lambda \left(\norm{w_{1i}}^2 +\norm{w_{2i}}^2 \right) + \theta_B\theta_{1i} \bigg((\eta+\kappa) \norm{w_{1i}}  + \eta \norm{w_{2i}} - R\bigg) - \theta_{1i}h_{1i}.
	\end{align*}
	In $\mathcal{D}_U$ we have
	\begin{align*}
	\norm{w_{ji}} \leq \frac{R_m}{4(\eta+\kappa)},\,\, j\in\{1,2\}
	\end{align*}
	and obtain that
	\begin{align*}
	\theta_B \theta_{1i}  \bigg( (\eta+\kappa)\norm{w_{1i}}+ \eta \norm{w_{2i}} -R \bigg)  \leq \theta_B\theta_{1i} \left( \frac{R_m}{2} -R\right) \leq \theta_B\theta_{1i} \frac{-R_m}{2},
	\end{align*}
	using $R \geq R_m>0$ from \eqref{eq:Rm_neg_def}.
	Then, the Lie derivative is bounded as follows:
	\begin{equation}\label{eq:Lie_der_bnd_U}
	\dot \Lambda_U \leq -\lambda \left(\norm{w_{1i}}^2 +\norm{w_{2i}}^2 \right) -\theta_B \theta_{1i} \frac{R_m}{2}  - \theta_{1i} h_{1i} 
	\leq \underbrace{-\lambda \left(\norm{w_{1i}}^2 +\norm{w_{2i}}^2 \right) -\theta_B \theta_{1i} \frac{R_m}{2}}_{\equalhat W_U(x_U, \theta_B)} \leq  0.
	\end{equation}		
	In deriving \eqref{eq:Lie_der_bnd_U}, we used that for $\theta_{1i} \in [0,1$], it is $h_{1i}\geq 0$ from \eqref{eq:hi_def} and therefore $\theta_{1i}h_{1i}\geq 0$. However, $\theta_B=0$ leads to $\dot\Lambda_U=0$ with possibly $\theta_{1i}\neq 0$.
	In the latter case and under the assumption that the dynamics \eqref{eq:dyn_sys} are piecewise continuous in time and locally Lipschitz in $(x_U, \theta_B)$, uniformly in time, \cite[Theorem 8.4, p323.]{Khalil_nonlinear} applies and states that all solutions of \eqref{eq:dyn_sys} initialized in $\mathcal{D}_U$ are bounded and satisfy $W_U(x_U, \theta_B)\rightarrow 0$. In turn, this implies that $(x_U, \theta_B)$ approaches the set
	\begin{align*}
	&E_U = \{x_U \in \mathcal{D}_U, \theta_B\in [0,1] \mid W_U(x_U,\theta_B)=0\} \\
	&= \{x_U\in \mathcal{D}_U,\, \theta_B\in [0,1] \mid w_{1i}=0,\,w_{2i}=0,\, \theta_{1i}\cdot\theta_B=0.\}
	\end{align*}
	We note that a set of the form $\{(x_U,\theta_B) \in B_r \mid V_2(x_U,\theta_B) \leq \min_{\norm{(x_U,\theta_B)}=r} V_1(x_U,\theta_B)\}$ with the ball $B_r$, $r>0$ in \cite[Theorem 8.4, p.323]{Khalil_nonlinear} imposes no further restriction on the initial condition in our case, since $V_1(x_U,\theta_B)= V_2(x_U,\theta_B) = \Lambda_U(x_U,\theta_B)$.
\end{proof}

\newpage

\end{document}